\documentclass{article}

    \PassOptionsToPackage{numbers, compress}{natbib}

\usepackage[preprint]{neurips_2026}

\usepackage[utf8]{inputenc}
\usepackage[T1]{fontenc}
\usepackage{microtype}

\usepackage{amsmath} 
\usepackage{amssymb}
\usepackage{amsfonts}
\usepackage{amsthm}
\usepackage{nicefrac}
\usepackage{bbm}

\usepackage[colorlinks=true, urlcolor=black]{hyperref}
\usepackage{url}

\usepackage[table, dvipsnames]{xcolor}

\usepackage{graphicx}
\graphicspath{{../}{../figs/}{../assets/}{../Figures/}}
\usepackage{rotating}
\usepackage{lscape}
\usepackage{adjustbox}
\usepackage{caption}
\usepackage{subcaption}
\usepackage{float}

\usepackage{booktabs}
\usepackage{colortbl}     
\usepackage{multirow}
\usepackage{siunitx}
\usepackage{nicematrix}
\usepackage{array}
\usepackage{tabularx}
\usepackage{arydshln}     
\usepackage{threeparttable} 
\usepackage{listings}

\usepackage[toc,page]{appendix}

\usepackage{pifont}       

\usepackage{tikz}
\usetikzlibrary{matrix,chains,positioning,arrows, calc}

\usepackage{wrapfig}
\usepackage{lipsum}

\definecolor{oxfordblue}{rgb}{0.0, 0.13, 0.28}
\definecolor{linkblue}{rgb}{0.0, 0.20, 0.40}
\definecolor{limegreen}{rgb}{0.2, 0.8, 0.2}
\definecolor{bondiblue}{rgb}{0.0, 0.58, 0.71}
\definecolor{britishracinggreen}{rgb}{0.0, 0.26, 0.15}
\definecolor{faintgray}{RGB}{245,245,245}
\definecolor{faintborder}{RGB}{230,230,230}
\definecolor{lightblack}{gray}{0.4}
\definecolor{pinegreen}{rgb}{0.0, 0.47, 0.44}
\definecolor{codegreen}{rgb}{0,0.6,0}
\definecolor{codegray}{rgb}{0.5,0.5,0.5}
\definecolor{codepurple}{rgb}{0.58,0,0.82}
\definecolor{backcolour}{rgb}{0.97,0.97,0.97}
\colorlet{darkgreen}{green!45!black}
\colorlet{darkred}{red!70!black}
\colorlet{limegreen}{LimeGreen}

\definecolor{bleudefrance}{rgb}{0.19, 0.55, 0.91}
\definecolor{darkcyan}{rgb}{0.0, 0.55, 0.55}
\definecolor{darkspringgreen}{rgb}{0.09, 0.45, 0.27}
\definecolor{darkmidnightblue}{rgb}{0.0, 0.2, 0.4}
\definecolor{MidnightBlue}{RGB}{25,25,112}
\definecolor{MidnightBlueComplementingGreen}{RGB}{25,112,25}
\definecolor{MidnightBlueComplementingPurple}{RGB}{112,25,112}
\definecolor{MidnightBlueComplementingRed}{RGB}{112,25,69}
\definecolor{WowColor}{rgb}{.75,0,.75}
\definecolor{MildlyAlarming}{rgb}{0.85,0.25,0.1}
\definecolor{SubtleColor}{rgb}{0,0,.50}
\definecolor{antiquefuchsia}{rgb}{0.57, 0.36, 0.51}
\definecolor{fashionfuchsia}{rgb}{0.96, 0.0, 0.63}
\definecolor{jade}{rgb}{0.0, 0.66, 0.42}
\definecolor{caribbeangreen}{rgb}{0.0, 0.8, 0.6}
\definecolor{aquamarine}{rgb}{0.5, 0.8, 0.85}
\definecolor{lightseagreen}{rgb}{0.13, 0.7, 0.67}
\definecolor{attentioncolor}{RGB}{152,90,81}
\definecolor{burgred}{RGB}{40,3,22}
\definecolor{AnnieGreen}{RGB}{17,123,92}
\definecolor{Turquoise}{RGB}{64,224,208}
\definecolor{darkjade}{RGB}{0,122,84}
\definecolor{RedAlizarin}{rgb}{0.82, 0.1, 0.26}
\definecolor{Window1}{RGB}{92,150,31}
\definecolor{Window1dark}{RGB}{41,67,13}
\definecolor{Window2}{RGB}{255,168,28}
\definecolor{Window2dark}{RGB}{114,75,12}
\definecolor{Window3}{RGB}{255,96,33}
\definecolor{Window3dark}{RGB}{97,36,12}
\definecolor{InputColor}{RGB}{20,255,177}
\definecolor{InputColorlight}{RGB}{222,237,229}

\lstdefinestyle{pythonstyle}{
    backgroundcolor=\color{backcolour},
    commentstyle=\color{codegreen},
    keywordstyle=\color{blue},
    numberstyle=\tiny\color{codegray},
    stringstyle=\color{codepurple},
    basicstyle=\ttfamily\footnotesize,
    breakatwhitespace=false,
    breaklines=true,
    captionpos=b,
    keepspaces=true,
    numbers=left,
    numbersep=5pt,
    showspaces=false,
    showstringspaces=false,
    showtabs=false,
    tabsize=2,
    language=Python
}
\usepackage{titletoc}
\newcommand\DoToC{%
  \startcontents
  \printcontents{}{1}{\textbf{Appendix Contents}\vskip3pt\hrule\vskip5pt}
  \vskip3pt\hrule\vskip5pt
}

\theoremstyle{plain}
\usepackage[framemethod=TikZ]{mdframed}

\newcounter{theo}[section] 
\renewcommand{\thetheo}{\arabic{section}.\arabic{theo}}

\newcounter{lem}[section] 
\renewcommand{\thelem}{\arabic{lem}}

\newcounter{prf}[section]
\renewcommand{\theprf}{\arabic{section}.\arabic{prf}}

\newtheorem{theorem}{Theorem}
\newtheorem{corollary}{Corollary}[theorem]
\newtheorem{lemma}[theorem]{Lemma}
\newtheorem{proposition}[theorem]{Proposition}

\theoremstyle{remark}

\theoremstyle{definition}
\newtheorem{definition}{Definition}[section]

\usepackage{bbm}             
\usepackage[UKenglish]{isodate}
\usepackage{silence}
    \definecolor{darkcerulean}{rgb}{0.03, 0.27, 0.49}
    \definecolor{smokyblack}{rgb}{0.06, 0.05, 0.03}
    \definecolor{warmblack}{rgb}{0.0, 0.26, 0.26}
    \definecolor{cobalt}{rgb}{0.0, 0.28, 0.67}
    \definecolor{darkcobalt}{rgb}{0.1, 0.38, 0.77}
    \hypersetup{colorlinks=true,
                linkcolor=cobalt,
                urlcolor=darkcerulean,
                citecolor=cobalt}

\usepackage{fancyvrb}
    \usepackage{tikz-cd}
    \pgfdeclarelayer{edgelayer}
    \pgfdeclarelayer{nodelayer}
    \pgfsetlayers{edgelayer,nodelayer,main}
    \tikzstyle{new style 0}=[fill={rgb,255: red,255; green,94; blue,247}, draw=black, shape=circle]
    \tikzstyle{pointy}=[fill=white, draw=black, shape=circle]
    \tikzstyle{pointy}=[->]
\usepackage{fontawesome}

\usepackage{cleveref}
\usepackage{comment}

\makeatletter
\newcommand{\pushright}[1]{\ifmeasuring@#1\else\omit\hfill$\displaystyle#1$\fi\ignorespaces}
\newcommand{\pushleft}[1]{\ifmeasuring@#1\else\omit$\displaystyle#1$\hfill\fi\ignorespaces}
\makeatother

\renewcommand{\phi}{\varphi}

\newcounter{termcounter}
\renewcommand{\thetermcounter}{\Roman{termcounter}}
\crefname{term}{term}{terms}
\creflabelformat{term}{#2\textup{(#1)}#3}

\makeatletter
\def\term{\@ifnextchar[\term@optarg\term@noarg}
\def\term@optarg[#1]#2{%
  \textup{#1}%
  \def\@currentlabel{#1}%
  \def\cref@currentlabel{[][2147483647][]#1}%
  \cref@label[term]{#2}}
\def\term@noarg#1{%
  \refstepcounter{termcounter}%
  \textup{(\thetermcounter)}%
  \cref@label[term]{#1}}
\makeatother

\crefname{lemma}{lemma}{lemmata}
\Crefname{lemma}{Lemma}{Lemmata}
\crefname{assumption}{assumption}{assumptions}
\Crefname{assumption}{Assumption}{Assumptions}
\crefname{example}{Example}{Examples}
\crefname{proposition}{Proposition}{Proposition}

\usepackage{bm}

\DeclareMathAlphabet{\mathsfit}{\encodingdefault}{\sfdefault}{m}{sl}
\SetMathAlphabet{\mathsfit}{bold}{\encodingdefault}{\sfdefault}{bx}{n}

\newif\ifcomments
\commentstrue          

\ifcomments
    \newcommand{\todo}[1]{\textcolor{blue}{[\textbf{TODO}: #1]}}
    \newcommand{\RS}[1]{\textcolor{purple}{[\textbf{RS}: #1]}}
    \newcommand{\TB}[1]{\textcolor{teal}{[\textbf{TB}: #1]}}
    \newcommand{\BW}[1]{\textcolor{olive}{[\textbf{BW}: #1]}}    
\else
    \newcommand{\RS}[1]{}
    \newcommand{\TB}[1]{}
    \newcommand{\BW}[1]{}
    \newcommand{\todo}[1]{}
\fi

\renewcommand{\textcolor}[2]{#2}

\title{Universal Time Series Generation with Neural Controlled Differential Equations}

\author{%
  Torben Berndt$^{1,*}$
  \quad Elyes Farjallah$^{1,2,*}$
  \quad Leif Seute$^{1,3,4}$
  \quad Raeid Saqur$^{5,6,7}$\\[10pt]
  \textbf{Benjamin Walker}$^{6,\dagger}$
  \quad \textbf{Jan Stühmer}$^{1,2,\dagger}$\\[10pt]
  $^{1}$Heidelberg Institute for Theoretical Studies, Heidelberg, Germany\\
  $^{2}$IAR, Karlsruhe Institute of Technology, Karlsruhe, Germany\\
  $^{3}$Max Planck Institute for Polymer Research, Mainz, Germany\\
  $^{4}$IWR, Heidelberg University, Heidelberg, Germany\\
  $^{5}$Dept. of Computer Science, University of Toronto, Toronto, Canada\\
  $^{6}$Mathematical Institute, University of Oxford, Oxford, UK\\
  $^{7}$Vector Institute, Toronto, Canada\\[0.4em]
  {\normalsize $^{*}$Equal first authorship. \quad
  $^{\dagger}$Equal senior authorship.}
}

\begin{document}
\maketitle


\begin{abstract}
  Recent work on the sequence universality of State Space Models (SSMs) has introduced efficient, maximally expressive continuous-time approaches for time-series modelling. While these works focus on discriminative settings, we extend this perspective to generative time-series modelling by proving 
  that maximally expressive Structured Linear Controlled Differential Equations (SLiCEs) are universal time-series generators, in the sense that they can approximate the induced path laws of continuous causal pushforwards on compact latent sets in \(W_\infty\). Building on these theoretical results, we propose \textit{Generative SLiCEs (G-SLiCEs)}, a maximally expressive continuous-time model for flow matching on path-space. Empirically, we show that expressivity improves performance in probabilistic forecasting and downstream tasks, while retaining the advantages of continuous-time models such as generalising to arbitrary observation grids. This is particularly beneficial for irregular grids, where fixed-grid models often struggle.
  \footnote{The source code is available at: \url{https://github.com/hits-mli/gslices}.}.
\end{abstract}

\section{Introduction}\label{sec:intro}
In probabilistic forecasting, capturing the full conditional distribution over future events is crucial, for example, when predicting extreme weather phenomena~\citep{price2025probabilistic}, electricity demand~\citep{taieb2015probabilistic}, or traffic~\citep{qian2023uncertainty}. Probabilistic forecasting is therefore inherently generative, and recent work has increasingly applied generative machine learning methods, such as diffusion models and flow matching~\citep{sohl2015deep, lipman2023flow, tong2023improving}, to this task~\citep{rasul2021autoregressive, tashiro2021csdi, alcaraz2022diffusion, bilovs2023modeling, kollovieh2023predict, kollovieh2024flow}. However, these models typically rely on sequence-model backbones that are not maximally expressive, which can limit their performance.

A parallel line of work develops expressive, parallel-in-time state space models (SSMs) by imposing structures on transition matrices, including block diagonal~\citep{fan2024advancing} and diagonal-plus-low-rank constructions~\citep{yang2024parallelizing,yang2024improving,siems2025deltaproduct}. 
Recent work unifies these approaches in a continuous-time framework building on Linear Controlled Differential Equations and characterise their universality~\cite{cirone2024deepSSM, walker2025structuredlinearcdesmaximally}.

We extend this theory to generative settings, where distributions on path space are modelled as pushforwards of latent path laws. Building on this perspective, we generalise current flow-based approaches for probabilistic forecasting from sequence-to-sequence mappings to path-to-path mappings. To address both the expressivity gap and the discrete-time formulation of existing state-of-the-art models, we propose Generative Structured Linear Controlled Differential Equations (G-SLiCEs): a maximally expressive continuous-time model for generative time-series modelling.

\paragraph{Main contributions}
    

    

We introduce a new generative approach for time-series data -- Generative Structured Linear Controlled Differential Equations (G-SLiCEs) -- that achieves state-of-the-art performance and is robust under grid shifts.
This relies on two key insights: \textbf{(1)} we formulate flow matching on path space using maximally expressive Structured Linear CDE backbones, and \textbf{(2)} we prove that pathwise expressivity implies universality for continuous causal pushforwards of
compactly supported latent path laws, linking discriminative expressivity of causal
path-to-path models to generative capabilities (Theorem \ref{thm:pathwise_to_generative}).
We demonstrate that G-SLiCEs outperform current state-of-the-art models on a comprehensive probabilistic forecasting benchmark, which we attribute to their universality, and show that the continuous-time formulation improves robustness to sampling-grid shifts
Further, we provide a concrete example showing that transition structure has provable expressivity consequences in the generative setting.

\subsection{Related work}
\label{subsec:related_work}

\paragraph{Flow matching}
Flow matching~\citep{lipman2023flow, albergo2023stochastic} has been proposed as powerful generative modelling technique on continuous, finite-dimensional Euclidean spaces and has since been generalised to Riemannian manifolds~\citep{chen2024flow}, discrete data~\citep{gat2024discrete} and functionals~\citep{kerrigan2023functional}.
These approaches do not, to our knowledge, yield distributional universality results on path space.

\paragraph{Generative models for probabilistic time-series forecasting.} 
Classical neural baselines either parameterise the predictive distribution explicitly (DeepAR~\citep{salinas2020deepar}, TFT~\citep{lim2021temporal}, autoregressive flows over RNN backbones~\citep{rasul2021autoregressive}) or are deterministic and combined with quantile heads (WaveNet~\citep{oord2016wavenet}, PatchTST~\citep{Yuqietal-2023-PatchTST}, DLinear~\citep{zeng2023transformers}). 
Generative approaches replace this distributional head with a learned density model. Variational auto-encoders (VAEs) such as TimeVAE~\citep{desai2021timevae} were explored first, followed by diffusion models, including CSDI~\citep{tashiro2021csdi}, SSSD~\citep{lopezalcaraz2022diffusionbased}, TSDiff~\citep{kollovieh2023predict}, and the stochastic-process formulation of \citet{bilovs2023modeling}. Most recently, TSFlow~\citep{kollovieh2024flow} advances this line of work by employing flow matching.

\paragraph{Expressivity of state space models.}

Recent work aims to make parallel-in-time sequence models more expressive while preserving efficiency. One direction parallelises non-linear RNNs by casting recurrence as a fixed-point problem solved with Newton-type methods \citep{lim2024parallelizing,gonzalez2024towards}. Another develops input-dependent linear RNNs and state-space models, including input-dependent block-diagonal LRNNs \citep{fan2024advancing}, DeltaNet variants \citep{yang2024parallelizing,yang2024improving,siems2025deltaproduct}, Mamba \citep{gu2023mamba}, Mamba-2 \citep{dao2024transformers}, RWKV-7 \citep{peng2025rwkv}, HGRN-2 \citep{qin2024hgrn2}, mLSTM \citep{beck2024xlstm}, Gated Linear Attention \citep{yang2024gated}, Gated Random Feature Attention \citep{peng2021random}, Gated Slot Attention \citep{zhang2024gated}, TTT-Linear \citep{sun2025learning}, and Titans \citep{behrouz2024titans}. These models usually use diagonal or diagonal-plus-low-rank transition matrices.


\section{Universality in the context of time series generation}\label{sec:method}

\begin{figure}[t]
    \centering
    \includegraphics[width=1.0\linewidth]{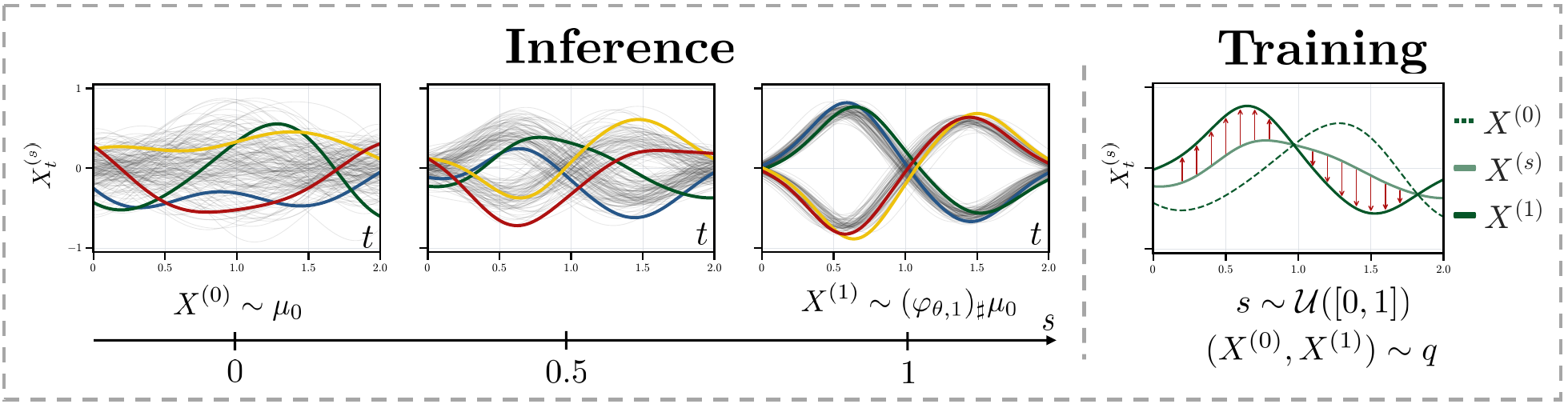}
    \caption{
    \textbf{Path-space flow matching with G-SLiCE.} G-SLiCE models probabilistic time-series generation as a continuous flow on path space. At inference time, an initial path \(X^{(0)} \sim \mu_0\) is transported by the learned flow \(\varphi_{\theta,s}\), producing a terminal sample \(X^{(1)} \sim (\varphi_{\theta,1})_{\#}\mu_0\). Each inference panel shows samples from the evolving distribution on path space.
    During training, the SLiCE vector field is trained to match the path-space displacement at the interpolated path \(X^{(s)}\), which is constructed from prior and data paths sampled from the joint distribution $q$.
    }
    \label{fig:g_slices_overview}
\end{figure}


We begin with the path-to-path model classes used in this paper. We then define the distributional notion of expressivity used for time series generation and show how it applies to G-SLiCEs.

\subsection{Linear NCDEs and SLiCEs}
\label{sec:linear_controlled_differential_equations}

Fix a time interval $[t_0,t_n]$. For $d\in\mathbb{N}$, let $\mathcal{X}(d)=C^{1,0}([t_0,t_n],\mathbb{R}^{d})$ denote the space of absolutely continuous, time-augmented input paths used by the model, all starting from the same point, equipped with the $1$-variation topology. Let $\mathcal{Y}(d)=C([t_0,t_n],\mathbb{R}^{d})$ denote the output path space, equipped with the supremum metric $\rho_\infty(Y,\widetilde Y)=\sup_{t\in[t_0,t_n]}\|Y_t-\widetilde Y_t\|_2$. A map $T:\mathcal{X}(d_X)\to\mathcal{Y}(d_y)$ is causal if $X|_{[t_0,t]}=\widetilde X|_{[t_0,t]}$ implies $T(X)_t=T(\widetilde X)_t$ for every $t\in[t_0,t_n]$.

Let $X\in\mathcal{X}(d_X)$ denote an input path and let $\omega^X\in C^{1,0}([t_0,t_n],\mathbb{R}^{d_\omega})$ denote a deterministic causal augmentation of $X$. In applications, $\omega^X$ may include time, observed values, masks, lags, and other deterministic features. Throughout this subsection, integrals are understood in the Riemann--Stieltjes sense, which is well defined for the absolutely continuous controls used by the model.

\begin{definition}[NCDE]
\label{def:neural_cde}
An NCDE is a causal path-to-path map $X\mapsto z^\theta(X)$ defined by
\begin{equation}
h_{t_0}=\xi_\phi(X_{t_0}),\qquad
h_t=h_{t_0}+\int_{t_0}^t g_\theta(h_s)\,d\omega^X_s,\qquad
z_t=r_\psi(h_t),
\end{equation}
where $h_t\in\mathbb{R}^{d_h}$, $z_t\in\mathbb{R}^{d_z}$, $\xi_\phi$ is a learnable initialisation map, $g_\theta(h)\in\mathbb{R}^{d_h\times d_\omega}$ is a learnable vector field, and $r_\psi$ is a learnable readout.
\end{definition}

\begin{definition}[Linear NCDE]
\label{def:linear_ncde}
A Linear NCDE is an NCDE whose vector field is linear in the hidden state. Equivalently, there are matrices $A_\theta^1,\ldots,A_\theta^{d_\omega}\in\mathbb{R}^{d_h\times d_h}$ such that
\begin{equation}
h_{t_0}=\xi_\phi(X_{t_0}),\qquad
h_t=h_{t_0}+\int_{t_0}^t\sum_{i=1}^{d_\omega}A_\theta^i h_s\,d\omega_s^{X,i},\qquad
z_t=r_\psi(h_t).
\end{equation}
\end{definition}

The dynamics are linear in $h$, but the induced map $X\mapsto z^\theta(X)$ can be highly non-linear, as the hidden state is multiplied by increments of the driving path. Linear NCDEs also benefit from an efficient parallel-in-time computation. On piecewise linear controls, \begin{equation}\label{eqn:recurrence_exp}
h_{t_{j+1}}=\Phi_j^\theta(X)h_{t_j},
\qquad
\Phi_j^\theta(X)=\exp\!\left(\sum_{i=1}^{d_\omega}A_\theta^i(\omega_{t_{j+1}}^{X,i}-\omega_{t_j}^{X,i})\right).
\end{equation}
Therefore $h_{t_k}=\Phi_{k-1}^\theta(X)\cdots\Phi_0^\theta(X)h_{t_0}$, and after computing the exact interval transition operators $\Phi_j^\theta(X)$ in parallel, evaluating the trajectory reduces to composing matrices. Since matrix multiplication is associative, the sequence of prefix products can be evaluated by a parallel scan.


Structured Linear CDEs, or SLiCEs, balance expressivity and efficiency by restricting each $A^i_{\theta}$ to a prescribed structured matrix family. Dense Linear NCDEs sit at the expressive but computationally expensive end of this spectrum, while diagonal transition matrices give efficient models that are not maximally expressive.
Other structures, such as block-diagonal, diagonal-plus-low-rank, sparse, and Walsh--Hadamard families, retain maximal expressivity while reducing recurrent cost and parameter count.
If the structure is also closed under matrix multiplication, as in fixed block-diagonal matrices, then the transition products remain structured, improving the efficiency of parallel-in-time evaluation.

\begin{theorem}[Path-to-path universality of maximally expressive SLiCEs]
\label{thm:slice_path_to_path_universal}
Let $\mathcal{K}\subset\mathcal{X}(d_X)$ be compact,
$T:\mathcal{X}(d_X)\to\mathcal{Y}(d_y)$ be continuous and causal, and $\omega^X_s=X_s$. Any SLiCE
class that is path-to-point maximally expressive in the sense of
\citet{walker2025structuredlinearcdesmaximally} with a linear readout satisfies
the following property. For every $\varepsilon>0$, there exists a model in the
class, with hidden dimension $d_h\in\mathbb{N}$ and feed-forward neural network
readout $r_\psi$, such that
\begin{equation}
\sup_{X\in\mathcal{K}}
\rho_\infty\bigl(z^\theta(X),T(X)\bigr)
\leq
\varepsilon .
\end{equation}
\end{theorem}

The proof is given in Appendix~\ref{app:proof_theorem_slice_path_to_path_universal}. Section~\ref{subsec:universal_time_series_generation} lifts this deterministic
approximation result to a distributional statement for time series generation.

\subsection{Universal time series generation}
\label{subsec:universal_time_series_generation}

We now define the distributional notion of expressivity used in this paper. If $X\sim\mu$ and $F_\theta:\mathcal{X}(d_X)\to\mathcal{Y}(d_y)$ is Borel measurable, then $(F_\theta)_\#\mu$ denotes the law of $F_\theta(X)$. For probability measures $\nu,\widetilde{\nu}$ on $\mathcal{Y}(d_y)$, the $\infty$-Wasserstein distance with respect to $\rho_\infty$ is
\begin{equation}
W_\infty(\nu,\widetilde{\nu})=\inf_{\pi\in\Pi(\nu,\widetilde{\nu})}\operatorname*{ess\,sup}_{(Y,\widetilde Y)\sim\pi}\rho_\infty(Y,\widetilde Y),
\end{equation}
where $\Pi(\nu,\widetilde{\nu})$ is the set of couplings of $\nu$ and $\widetilde{\nu}$.

\begin{definition}[Universal causal time series generator]
\label{def:universal_time_series_generator}
Let
\begin{equation}
\mathcal{F}
=
\{F_\theta:\mathcal{X}(d_X)\to\mathcal{Y}(d_y)\mid \theta\in\Theta\}
\end{equation}
be a class of Borel measurable causal maps. We say that $\mathcal{F}$ is a universal causal time series generator if, for every compact set $\mathcal{K}\subset\mathcal{X}(d_X)$, every Borel probability measure $\mu$ satisfying $\mu(\mathcal{K})=1$, every continuous causal map $T:\mathcal{X}(d_X)\to\mathcal{Y}(d_y)$, and every $\varepsilon>0$, there exists $\theta\in\Theta$ such that
\begin{equation}
W_\infty\bigl((F_\theta)_\#\mu,T_\#\mu\bigr)\leq\varepsilon,
\end{equation}
where $W_\infty$ is computed on $\mathcal{Y}(d_y)$ with respect to $\rho_\infty$.
\end{definition}

\begin{theorem}[Pathwise approximation implies generative approximation]
\label{thm:pathwise_to_generative}
Let
\begin{equation}
\mathcal{F}
=
\{F_\theta:\mathcal{X}(d_X)\to\mathcal{Y}(d_y)\mid \theta\in\Theta\}
\end{equation}
be a class of Borel measurable causal maps. Suppose that, for every compact set $\mathcal{K}\subset\mathcal{X}(d_X)$, every continuous causal map $T:\mathcal{X}(d_X)\to\mathcal{Y}(d_y)$, and every $\varepsilon>0$, there exists $\theta\in\Theta$ such that
\begin{equation}
\sup_{X\in\mathcal{K}}
\rho_\infty\bigl(F_\theta(X),T(X)\bigr)
\leq\varepsilon .
\end{equation}
Then $\mathcal{F}$ is a universal causal time series generator.
\end{theorem}

The proof is given in Appendix~\ref{app:theory}. It uses the coupling obtained by evaluating $F_\theta$ and $T$ on the same input path $X\sim\mu$.

\begin{corollary}[Maximally expressive SLiCEs are universal causal time series generators]
\label{cor:slice_universal_generator}
Any SLiCE class satisfying Theorem~\ref{thm:slice_path_to_path_universal} is a
universal causal time series generator. 
\end{corollary}

Corollary~\ref{cor:slice_universal_generator} applies to the block-diagonal, diagonal-plus-low-rank, sparse, and Walsh--Hadamard SLiCEs
\citep{walker2025structuredlinearcdesmaximally,movahedi2025fixedpointrnnsdiagonaldense}. However, not every efficient state-space structure has this expressivity. As shown by \citet{cirone2024deepSSM}, S4 corresponds to a non-selective Linear NCDE whose transition is driven only by the time channel $\omega_t=t$, while Mamba corresponds to a selective SLiCE whose driving path depends on $X$ but whose transition matrices are diagonal. These restrictions enable more efficient parallel-in-time evaluation, but they also limit expressivity. Indeed, single-layer diagonal SLiCEs, S4, and Mamba with linear readouts are not universal even for terminal-time path-to-point functions \citep{cirone2024deepSSM}. Universality results are available for S4D-style recurrences with nonlinear projections and for stacked SSMs with layer-wise nonlinearities \citep{orvieto2024universality,wang2023state}, but these are discrete sequence-to-sequence approximation theorems rather than the continuous causal path-to-path universality considered here. The next section gives a simple state-tracking example that makes the distinction between these transition structures explicit.

\subsection{Example: expressivity gap}
\label{subsec:example_expressivity_gap}

\begin{figure}[ht]
    \centering

    \begin{minipage}[c]{0.48\linewidth}
        \centering
        \includegraphics[width=\linewidth]{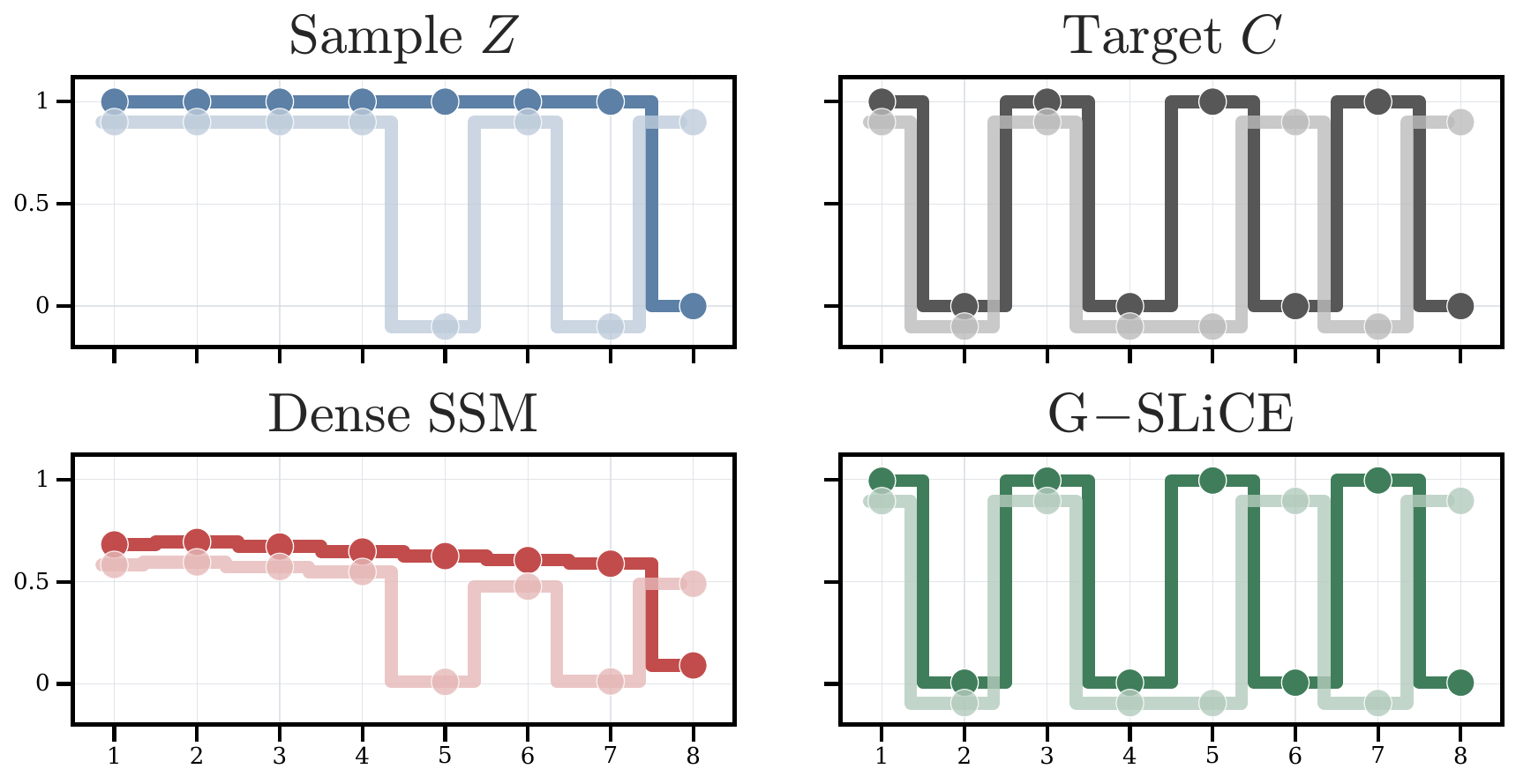}
    \end{minipage}
    \hfill
    \begin{minipage}[c]{0.48\linewidth}
        \centering
        \includegraphics[width=\linewidth]{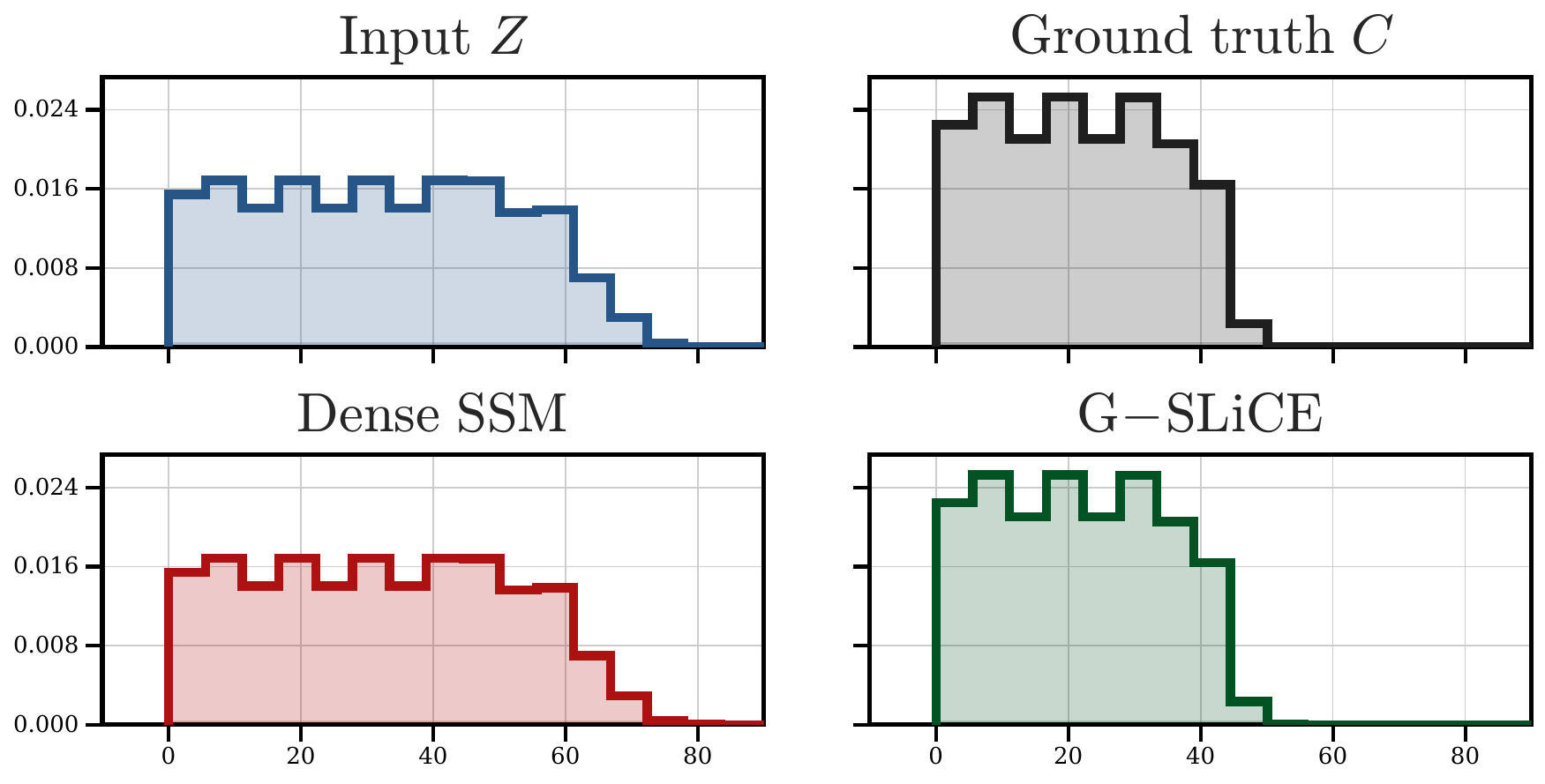}
    \end{minipage}

    \caption{
    Expressivity gap on the sequence task (Section~\ref{subsec:example_expressivity_gap}). Left: for representative Bernoulli input sequences \(Z\), the target map \(C\) removes consecutive ones by the recursion in \eqref{eq:hardcore_target}. The dense non-selective SSM fails to learn this state-dependent rule, while G-SLiCE reproduces the target sequence. Right: empirical distributions of cumulative sums over generated sequences. G-SLiCE closely matches the ground-truth pushforward law \(\mu_{n,p}\), whereas the dense non-selective SSM produces a visibly distorted distribution.}
    \label{fig:hardcore-results}
\end{figure}

We illustrate how transition structure affects a finite state-tracking task. Let
\(\mathcal H_n\) be the set of binary sequences with no consecutive ones. For
\(Z_1,\ldots,Z_n\stackrel{\mathrm{i.i.d.}}{\sim}\operatorname{Bernoulli}(p)\), define
\begin{equation}
\label{eq:hardcore_target}
C_1=Z_1,
\qquad
C_k=Z_k(1-C_{k-1}),
\qquad
2\leq k\leq n .
\end{equation}
Then \(C\in\mathcal H_n\) almost surely. We write \(\mu_{n,p}\) for the induced law.

We compare zero-order-hold exact-flow discretisations of continuous-time linear
state-space models,
\begin{equation}
\label{eq:selective_recurrence}
h_k=\exp(A(Z_k))h_{k-1}+\beta(Z_k),
\qquad
\widehat C_k=w^\top h_k+b,
\end{equation}
where \(A\) and \(\beta\) are affine maps and \(h_0\) is fixed independently of the
input sequence. The non-selective case, where \(A\) is constant, abstracts
S4-style transitions. The diagonal selective case, where \(A(Z_k)\) is diagonal
and input-dependent, abstracts Mamba-style selective diagonal transitions. Dense
selective transitions correspond to Linear NCDEs.


\begin{proposition}
\label{prop:hard_core_gap_distribution}
Fix \(n\geq1\). For every \(p\in[0,1]\) and every \(\varepsilon>0\), there exists a width \(2\) dense selective exact-flow SSM such that its induced output law \(\widehat\mu_{n,p}\) satisfies
\[
W_\infty(\widehat\mu_{n,p},\mu_{n,p})<\varepsilon .
\]
If \(n\ge2\), then no dense non-selective exact-flow SSM can approximate
\(\mu_{n,1/2}\) with error strictly less than \(1/4\).
If \(n\ge d+2\), then no width \(d\) diagonal selective exact-flow SSM can approximate
\(\mu_{n,1}\) with error strictly less than \(1/2\).
\end{proposition}

Figure~\ref{fig:hardcore-results} shows the empirical effect. The dense
selective SLiCE tracks the state $C_k$ and closely matches the pushforward
law \(\mu_{n,p}\), while the dense non-selective SSM distorts the distribution.
The proof of Proposition~\ref{prop:hard_core_gap_distribution} is given in
Appendix~\ref{app:hardcore_gap}.

\section{Generative Structured Linear CDEs}
\label{ssec:gslice}

We now instantiate the expressivity result of Corollary~\ref{cor:slice_universal_generator} in a concrete generative model class, which we call \textit{Generative Structured Linear CDEs (G-SLiCEs)}. A G-SLiCE consists of a causal SLiCE map $G_\theta:\mathcal{X}(d_X)\to\mathcal{Y}(d_y)$ and a prior distribution $\mu$. Sampling from the generative model means drawing \(X\sim\mu\) and returning \(G_\theta(X)\), so the generated law is $\nu_\theta = (G_\theta)_\#\mu $. This is exactly the pushforward setting of Definition~\ref{def:universal_time_series_generator}, and hence G-SLiCEs are universal causal time series generators. Thus, G-SLiCEs inherit continuous path-space distributional universality while retaining the parallel-in-time and continuous-time structure of SLiCEs.

Our practical implementation of G-SLiCEs uses the flow matching construction of \citet{kollovieh2024flow}, but generalised from grids to paths. The implementation defines three components, a conditional prior law on path space, a deterministic flow on path space, and a training mechanism. To avoid confusion with physical (path-parametrising) time $t$, we denote the flow-matching-time $s$ by bracketed superscripts. Given an initial path \(X^{(0)}\sim\mu\), the model defines a path-valued flow
\begin{equation}
    \frac{d}{ds}X^{(s)}
    =
    F_\theta\!\left(s,X^{(s)}\right),
    \qquad
    X^{(0)}\sim\mu,
    \qquad
    s\in[0,1],
\end{equation}
where
\(
    F_\theta:
    [0,1]\times \mathcal X(d_X)
    \to
    \mathcal X(d_X)
\)
is a causal SLiCE vector-field network. The terminal path \(X^{(1)}\) induces
the generated law
\begin{equation}
    \nu_\theta = (\phi_{\theta,1})_\#\mu ,
\end{equation}
where \(\phi_{\theta,1}\) denotes the flow map from \(s=0\) to \(s=1\).
When the target output space differs from the model input path space, the terminal path is followed by the corresponding readout or projection.
Figure \ref{fig:g_slices_overview} visualises this process in its right panel.

The flow time \(s\) is injected by concatenating it as an additional constant
channel to the current input path. The vector-field network \(F_\theta\) is
parameterised by a SLiCE backbone, so the learned velocity field preserves
causality along the physical time axis and uses the same exact-flow and
parallel-in-time structure described above.

\paragraph{Prior.} Following \citet{kollovieh2024flow}, we use Gaussian processes (GPs)
\citep{Rasmussen2006Gaussian} as the noise distribution \(\mu\), since they
provide a canonical non-parametric model distribution on path space. The
construction of the GP depends on the use case.
In the unconditional setting, we choose an unfitted Gaussian-process prior $\mu = \mathcal{GP}(m,k)$ on path space.
In the conditional setting, each target trajectory \(X^{(1)}\) determines conditioning information \(C\), such as a partially observed prefix $C = X^{(1)}\big \vert_{[t_0,t_c]}$ or a finite set of observations $C=\{(t_i,X^{(1)}_{t_i})\}_{i=0}^m$.
Denoting the posterior means and kernels of a GP fitted on \(C\) by \(m_{\mathrm{post}}\) and \(k_{\mathrm{post}}\), we use the conditional prior law
\(
    \mu(\cdot\mid C)
    =
    \mathcal{GP}
    \bigl(
        m_{\mathrm{post}}(\cdot\mid C),
        k_{\mathrm{post}}(\cdot,\cdot\mid C)
    \bigr).
\)
Conditional generation is performed by sampling $X^{(0)}\sim\mu(\cdot\mid C)$ and returning the terminal path \(X^{(1)}\).
The resulting conditional model law is
\(
    \nu_\theta(\cdot\mid C)
    =
    (\phi_{\theta,1}^{C})_\#
    \mu(\cdot\mid C),
\)
where \(\phi_{\theta,1}^{C}\) denotes the flow map with context \(C\).

\paragraph{Training.} We train the vector-field network using conditional flow matching \citep{lipman2023flow}.
In the unconditional case, we independently sample data paths \(X^{(1)}\) and noise paths \(X^{(0)}\) from the corresponding prior and match them using the mini-batched optimal transport coupling $q(X^{(0)}, X^{(1)})$ \cite{tong2023improving}.
In the conditional case, each data path \(X^{(1)}\sim\nu\) determines a context $C=\Gamma(X^{(1)})$, and we sample $X^{(0)}\sim\mu(\cdot\mid C)$.
Equivalently,
\[
    q(X^{(0)},X^{(1)})
    =
    \nu(X^{(1)})
    \mu\bigl(X^{(0)}\mid \Gamma(X^{(1)})\bigr).
\]
For the straight-line interpolant $X^{(s)}=(1-s)X^{(0)} + sX^{(1)}$, the target velocity is $u_s(X^{(0)},X^{(1)})=X^{(1)}-X^{(0)}$.
The training objective is
\[
\mathcal L(\theta)
=
\mathbb E_{\substack{
s\sim\mathcal U[0,1]\\
(X^{(0)},X^{(1)})\sim q
}}
\left[
    \rho_2 \left(
    F_\theta
    \bigl(
        s,
        X^{(s)}
    \bigr),
    X^{(1)}-X^{(0)}
    \right)^2
\right].
\]
In implementation, the metric is evaluated on the discretised physical-time grid. We visualise the training process in the right panel of Figure \ref{fig:g_slices_overview}.

There are two technical considerations necessary to connect the universality statement in Section~\ref{subsec:universal_time_series_generation} to the flow-matching implementation used in our experiments. First, Corollary~\ref{cor:slice_universal_generator} is stated for the case where the G-SLiCE generator is the direct map from source to target distribution. However, any G-SLiCE generator can be realised as an augmented path-space flow, as shown in Appendix~\ref{app:g-slice-as-path-flow}. Second, the Gaussian process input laws used in practice are not compactly supported. However, after discretisation, interpolation, and augmentation, they assign arbitrarily high probability to compact subsets of \(\mathcal{X}(d_X)\). Hence the same coupling argument gives a high-probability form of the guarantee. Full details are given in Appendix~\ref{app:theory}.

\section{Experiments}\label{sec:experiments}

We evaluate G-SLiCE as a generative time-series model in three regimes:
First, we study conditional probabilistic forecasting and unconditional generation, in Sections~\ref{sec:experiments_probabilistic_forecasting} and~\ref{sec:experiments_unconditional_generation}. 
Second, we test whether the continuous-time construction improves robustness under changes in the sampling grid in Section~\ref{sec:experiments_generalising}. We consider two such distribution shifts: out-of-distribution sampling frequencies and irregularly observed grids. Experimental details, hyperparameter ranges, and compute resources are given in Appendix~\ref{app:experimental_details}.

\paragraph{Datasets.}
We use univariate datasets from the GluonTS benchmark suite \citep{gluonts_arxiv,gluonts_jmlr}: Electricity \citep{dua2017uci}, Exchange \citep{lai2018modeling}, KDDCup \citep{godahewa2monash}, M4-Hourly \citep{makridakis2020m4}, Solar \citep{lai2018modeling}, Traffic \citep{dua2017uci}, UberTLC-Hourly \citep{fivethirtyeight2016}, Wikipedia \citep{gasthaus2019probabilistic}, and ETTSmall \citep{haoyietal-informer-2021}.
The high-frequency ETTSmall dataset is used for the grid-shift experiments; the other datasets are used for probabilistic forecasting and unconditional generation.

\paragraph{Baselines.}
We compare G-SLiCE against four groups of baselines. The first group contains classical statistical forecasting methods: Seasonal Naive (SN), AutoARIMA, and AutoETS \citep{hyndman2008forecasting}. The second group contains neural forecasting models: DLinear \citep{zeng2023transformers}, DeepAR \citep{salinas2020deepar}, Temporal Fusion Transformer (TFT) \citep{lim2021temporal}, WaveNet \citep{oord2016wavenet}, and PatchTST \citep{Yuqietal-2023-PatchTST}. The third group contains diffusion-based generative models: CSDI \citep{tashiro2021csdi}, SSSD \citep{lopezalcaraz2022diffusionbased}, the model of Bilo\v{s} et al.~\citep{bilovs2023modeling}, and TSDiff \citep{kollovieh2023predict}. The fourth group contains the flow-based model TSFlow \citep{kollovieh2024flow}. For probabilistic forecasting, we compare against all baselines. For the remaining experiments, we focus on TSFlow as the strongest and conceptually closest flow-based baseline.

\subsection{Probabilistic forecasting}\label{sec:experiments_probabilistic_forecasting}

\begin{table*}
    \centering
    \scriptsize
    \setlength{\tabcolsep}{2.6pt}
    \renewcommand{\arraystretch}{0.92}
    \caption{Comparison of statistical baselines, neural forecasting models, diffusion-based generative models, and flow-based generative models against G-SLiCE on conditional probabilistic forecasting. We report mean and standard deviation of CRPS over five random seeds on eight GluonTS datasets \citep{gluonts_arxiv,gluonts_jmlr}. Best scores are in \textbf{bold}; second-best scores are \underline{underlined}. The last row indicates whether the selected G-SLiCE configuration uses a dense state-dependent transition block.}
    \label{tab:conditional-generation-results}
    \resizebox{\textwidth}{!}{%
    \begin{tabular}{lcccccccc}
        \toprule
        Method & Electr. & Exch. & KDD & M4-H & Solar & Traffic & Uber & Wiki \\
        \midrule
        SN
        & $0.069_{\pm 0.000}$ & $0.013_{\pm 0.000}$ & $0.561_{\pm 0.000}$ & $0.048_{\pm 0.000}$
        & $0.512_{\pm 0.000}$ & $0.221_{\pm 0.000}$ & $0.299_{\pm 0.000}$ & $0.423_{\pm 0.000}$ \\
        ARIMA
        & $0.344_{\pm 0.000}$ & $\underline{0.008_{\pm 0.000}}$ & $0.514_{\pm 0.000}$ & $0.031_{\pm 0.000}$
        & $0.558_{\pm 0.003}$ & $0.486_{\pm 0.000}$ & $0.478_{\pm 0.000}$ & $0.654_{\pm 0.000}$ \\
        ETS
        & $0.055_{\pm 0.000}$ & $\underline{0.008_{\pm 0.000}}$ & $0.584_{\pm 0.000}$ & $0.070_{\pm 0.000}$
        & $0.550_{\pm 0.000}$ & $0.492_{\pm 0.000}$ & $0.520_{\pm 0.000}$ & $0.651_{\pm 0.000}$ \\
        DLinear
        & $0.058_{\pm 0.001}$ & $0.015_{\pm 0.004}$ & $0.318_{\pm 0.015}$ & $0.055_{\pm 0.007}$
        & $0.794_{\pm 0.027}$ & $0.131_{\pm 0.000}$ & $0.250_{\pm 0.006}$ & $0.259_{\pm 0.002}$ \\
        \midrule
        DeepAR
        & $0.051_{\pm 0.000}$ & $0.013_{\pm 0.004}$ & $0.362_{\pm 0.017}$ & $0.045_{\pm 0.013}$
        & $0.429_{\pm 0.055}$ & $0.103_{\pm 0.002}$ & $0.168_{\pm 0.002}$ & $0.215_{\pm 0.003}$ \\
        TFT
        & $0.060_{\pm 0.001}$ & $\mathbf{0.007_{\pm 0.000}}$ & $0.543_{\pm 0.048}$ & $0.038_{\pm 0.002}$
        & $0.371_{\pm 0.006}$ & $0.128_{\pm 0.005}$ & $0.202_{\pm 0.009}$ & $0.219_{\pm 0.004}$ \\
        WaveNet
        & $0.058_{\pm 0.008}$ & $0.012_{\pm 0.001}$ & $0.305_{\pm 0.018}$ & $0.055_{\pm 0.014}$
        & $0.360_{\pm 0.009}$ & $0.099_{\pm 0.002}$ & $0.180_{\pm 0.013}$ & $\mathbf{0.207_{\pm 0.003}}$ \\
        PatchTST
        & $0.055_{\pm 0.001}$ & $0.010_{\pm 0.001}$ & $0.420_{\pm 0.011}$ & $0.034_{\pm 0.004}$
        & $0.728_{\pm 0.015}$ & $0.151_{\pm 0.007}$ & $0.219_{\pm 0.004}$ & $0.209_{\pm 0.001}$ \\
        \midrule
        CSDI
        & $0.051_{\pm 0.000}$ & $0.013_{\pm 0.001}$ & $0.309_{\pm 0.006}$ & $0.043_{\pm 0.004}$
        & $0.360_{\pm 0.006}$ & $0.152_{\pm 0.001}$ & $0.213_{\pm 0.007}$ & $0.318_{\pm 0.012}$ \\
        SSSD
        & $0.048_{\pm 0.001}$ & $0.010_{\pm 0.001}$ & $\mathbf{0.274_{\pm 0.009}}$ & $0.050_{\pm 0.007}$
        & $0.384_{\pm 0.023}$ & $0.097_{\pm 0.002}$ & $0.156_{\pm 0.007}$ & $0.209_{\pm 0.004}$ \\
        Bilo\v{s} et al.
        & $0.067_{\pm 0.002}$ & $0.012_{\pm 0.004}$ & $1.147_{\pm 0.300}$ & --
        & $0.379_{\pm 0.009}$ & $0.317_{\pm 0.053}$ & $0.450_{\pm 0.086}$ & $0.318_{\pm 0.022}$ \\
        TSDiff
        & $0.049_{\pm 0.000}$ & $0.011_{\pm 0.001}$ & $0.311_{\pm 0.026}$ & $0.036_{\pm 0.001}$
        & $0.358_{\pm 0.020}$ & $0.098_{\pm 0.002}$ & $0.172_{\pm 0.005}$ & $0.221_{\pm 0.001}$ \\
        \midrule
        TSFlow
        & $\underline{0.045_{\pm 0.000}}$ & $\underline{0.008_{\pm 0.000}}$ & $0.288_{\pm 0.004}$ & $\underline{0.028_{\pm 0.008}}$
        & $\underline{0.344_{\pm 0.006}}$ & $\underline{0.082_{\pm 0.000}}$ & $\underline{0.154_{\pm 0.002}}$ & $\mathbf{0.207_{\pm 0.001}}$\rlap{\textsuperscript{\dag}} \\
        G-SLiCEs
        & $\mathbf{0.044_{\pm 0.000}}$ & $\mathbf{0.007_{\pm 0.000}}$ & $\underline{0.275_{\pm 0.009}}$ & $\mathbf{0.023_{\pm 0.001}}$
        & $\mathbf{0.342_{\pm 0.012}}$ & $\mathbf{0.080_{\pm 0.000}}$ & $\mathbf{0.150_{\pm 0.002}}$ & $0.218_{\pm 0.002}$ \\
        \midrule
        Dense block
        & \ding{55} & \ding{51} & \ding{51} & \ding{51}
        & \ding{55} & \ding{55} & \ding{51} & \ding{51} \\
        \bottomrule
        \addlinespace[1pt]
        \multicolumn{9}{l}{\scriptsize{$^\dag\,$We were not able to reproduce the reported TSFlow result on Wiki2000; the best value we obtained was a CRPS of $0.218_{\pm 0.001}$.}}
    \end{tabular}%
    }
\end{table*}

Table~\ref{tab:conditional-generation-results} reports conditional probabilistic forecasting results in terms of \textit{continuous ranked probability score} (CRPS)~\citep{gneiting2007strictly}. G-SLiCE is competitive with the full set of statistical, neural, diffusion-based, and flow-based baselines. It obtains the best score on \(6/8\) datasets and outperforms TSFlow on \(7/8\) datasets. The last row indicates whether the best G-SLiCE configuration uses a dense state-dependent transition block. The selected structure varies by dataset, suggesting that dense transitions are useful only when the additional expressivity is needed. On the Wiki2000 dataset, we were unable to reproduce the TSFlow value reported by \citet{kollovieh2024flow} using their specified hyperparameters. We therefore report both the published TSFlow value and the best value obtained in our reproduction.

We confirm statistically significant rank differences with a global Friedman test. A paired Wilcoxon signed-rank test between G-SLiCE and TSFlow gives a $p$-value of $0.055$ including Wiki2000 and $p=0.008$ excluding it. Details are in Appendix~\ref{app:statistical_tests}. Since TSFlow is the strongest and closest flow-based baseline, subsequent experiments compare only G-SLiCE and TSFlow to isolate the effect of the SLiCE backbone.

\subsection{Unconditional generation}\label{sec:experiments_unconditional_generation}

To evaluate the unconditional generation capabilities of G-SLiCE, we sample input paths from an isotropic Gaussian process prior and train the model to generate sequences from the target data distributions. We assess the quality of the generated sequences by comparing the \(2\)-Wasserstein distance between \(10{,}000\) generated samples and \(10{,}000\) real samples for both G-SLiCE and TSFlow. The results are reported in Table~\ref{tab:unconditional_generation_w2_results}. Second, we evaluate the downstream predictive utility of the generated samples using the Linear Predictive Score (LPS). For this metric, we split \(10{,}000\) generated sequences into an observed context and a held-out future window, train a linear ridge regression model to predict the future from the context, and evaluate the resulting predictor on real test sequences. The LPS is defined as the test CRPS of this predictor. Results are shown in Table~\ref{tab:unconditional_generation_lps_results}.

\begin{table*}[h]
    \centering
    \small
    \setlength{\tabcolsep}{5pt}
    \caption{
    Unconditional generation quality measured by \(2\)-Wasserstein distance on eight GluonTS datasets. For each model, we generate \(10{,}000\) synthetic sequences and compare them against \(10{,}000\) real sequences from the corresponding dataset. Results are reported as mean \(\pm\) standard deviation over five random seeds. Best scores are in \textbf{bold}.
    }
    \label{tab:unconditional_generation_w2_results}
    \resizebox{\textwidth}{!}{%
        \begin{tabular}{lcccccccc}
            \toprule
            Method & Electr. & Exchange & KDDCup & M4 (H) & Solar & Traffic & UberTLC & Wiki2000 \\
            \midrule
            TSFlow & $2.050_{\pm 0.308}$ & $1.060_{\pm 0.366}$ & $\mathbf{7.192_{\pm 0.481}}$ & $3.687_{\pm 0.194}$ & $5.758_{\pm 0.125}$ & $3.971_{\pm 0.135}$ & $13.521_{\pm 2.936}$ & $\mathbf{13.685_{\pm 6.216}}$ \\
            G-SLiCE & $\mathbf{1.958_{\pm 0.264}}$ & $\mathbf{1.059_{\pm 0.384}}$ & $\mathbf{7.192_{\pm 0.481}}$ & $\mathbf{3.441_{\pm 0.213}}$ & $\mathbf{4.086_{\pm 0.280}}$ & $\mathbf{3.175_{\pm 0.163}}$ & $\mathbf{13.431_{\pm 3.052}}$ & $13.874_{\pm 6.394}$ \\
            
            \bottomrule
        \end{tabular}%
    }
\end{table*}

\begin{table*}[h]
    \centering
    \small
    \setlength{\tabcolsep}{5pt}
    \caption{
    Unconditional generation quality measured by Linear Predictive Score (LPS) on eight GluonTS datasets. For each model, we generate \(10{,}000\) synthetic sequences, split them into context and prediction windows, train a linear ridge predictor on the generated samples, and evaluate its test CRPS on real held-out sequences. Results are reported as mean \(\pm\) standard deviation over five random seeds. Best scores are in \textbf{bold}.
    }
    \label{tab:unconditional_generation_lps_results}
    \resizebox{\textwidth}{!}{%
        \begin{tabular}{lcccccccc}
            \toprule
            Method & Electr. & Exchange & KDDCup & M4 (H) & Solar & Traffic & UberTLC & Wiki2000 \\
            \midrule
            TSFlow & $0.131_{\pm 0.035}$ & $\mathbf{0.010_{\pm 0.000}}$ & $\mathbf{0.466_{\pm 0.015}}$ & $0.087_{\pm 0.007}$ & $0.744_{\pm 0.007}$ & $0.326_{\pm 0.002}$ & $0.473_{\pm 0.003}$ & $\mathbf{0.369_{\pm 0.006}}$ \\
            SLICE & $\mathbf{0.109_{\pm 0.008}}$ & $\mathbf{0.010_{\pm 0.000}}$ & $\mathbf{0.466_{\pm 0.015}}$ & $\mathbf{0.064_{\pm 0.009}}$ & $\mathbf{0.699_{\pm 0.010}}$ & $\mathbf{0.246_{\pm 0.001}}$ & $\mathbf{0.397_{\pm 0.004}}$ & $0.371_{\pm 0.009}$ \\
            \bottomrule
        \end{tabular}%
    }
\end{table*}

\subsection{Generalisation across sampling grids}\label{sec:experiments_generalising}

We test whether the continuous-time construction of G-SLiCE improves robustness to changes in the observation grid. We consider two grid shifts on ETTSmall15min: changes in sampling frequency and irregular subsampling. In all cases, context and prediction windows are fixed to \(24\) hours.

\begin{wrapfigure}[25]{r}{0.49\textwidth}
    \centering
    \vspace{-1.25em}
    \includegraphics[width=0.95\linewidth]{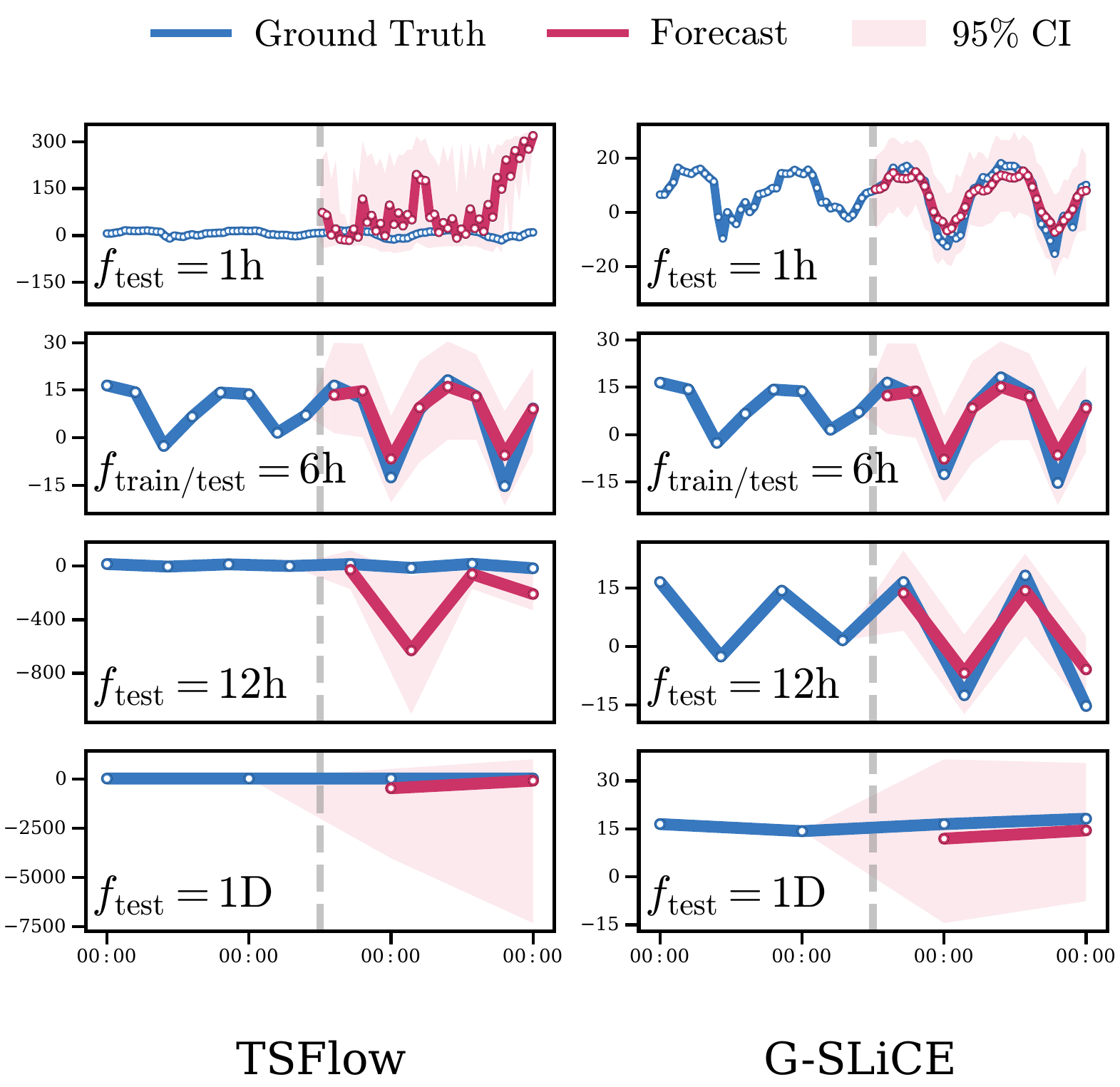}
    \caption{
    Representative cross-frequency forecasts on \textcolor{blue}{ETTSmall1h} for models trained at \(6\)-hour resolution and evaluated on different test frequencies. Curves show ground truth and predictive means; shaded regions indicate \(95\%\) confidence intervals. G-SLiCE remains stable across changes in the observation grid, while TSFlow is more sensitive to frequency shifts.
    }
    \label{fig:example_forecasts_cross_frequency}
    \vspace{-1em}
\end{wrapfigure}

\paragraph{Cross-frequency generalisation.}
We train models on one of four uniform grids: \(15\)-minute, hourly, \(6\)-hourly, or \(12\)-hourly, and evaluate each trained model on all four grids. G-SLiCE is evaluated directly on the requested time grid. For TSFlow, we report direct evaluation and two grid-mismatch adaptations: holding the latest value until a new observation is available (zero-order hold), and oversampling the driving Gaussian process to match the training grid.

Table~\ref{tab:ett15m_crossfreq} reports CRPS. G-SLiCE remains stable across most frequency shifts, whereas direct TSFlow evaluation can fail severely when the grids differ. The strongest example is the \(15\)-minute-train/\(12\)-hour-test setting: G-SLiCE obtains CRPS \(0.189\), while direct TSFlow deteriorates to \(906.205\). The repeat and GP-oversampling adaptations improve TSFlow in some cases, but their cost scales with the higher resolution of the train grid rather than with the requested evaluation grid, making them more computationally expensive. Furthermore, this route is only available when testing on grids coarser than those trained on. Figure~\ref{fig:example_forecasts_cross_frequency} shows representative forecasts for models trained at \(6\)-hour resolution. 
 
\begin{table*}[h]
    \centering
    \caption{Cross-frequency generalisation on ETTSmall15min. Each model is trained at the frequency indicated by the column and evaluated at the frequency indicated by the row. Entries report CRPS. G-SLiCE is evaluated directly on each physical time grid. For TSFlow, we report direct evaluation and two adaptation rules for grid mismatch: repeating the latest update and GP oversampling.}
    \label{tab:ett15m_crossfreq}
    \setlength{\tabcolsep}{4pt}
    \renewcommand{\arraystretch}{0.95}
    \scriptsize
    \resizebox{\textwidth}{!}{%
    \begin{tabular}{lcccc|cccc|cccc|cccc}
        \toprule
        & \multicolumn{4}{c|}{G-SLiCE} & \multicolumn{4}{c|}{TSFlow} & \multicolumn{4}{c|}{TSFlow (repeat)} & \multicolumn{4}{c}{TSFlow (GP oversample)} \\
        test $\backslash$ train & 15min & 1h & 6h & 12h & 15min & 1h & 6h & 12h & 15min & 1h & 6h & 12h & 15min & 1h & 6h & 12h \\
        \midrule
        15min & 0.204 & 0.215 & 0.675 & 0.459 & 0.205 & 0.768 & 24.622 & 7.319 & 0.205 &       &       &       & 0.205 &       &       &       \\
        1h    & 0.204 & 0.210 & 1.460 & 0.230 & 0.809 & 0.203 & 0.209  & 0.208 & 0.206 & 0.203 &       &       & 0.208 & 0.203 &       &       \\
        6h    & 0.205 & 0.208 & 0.216 & 0.220 & 328.794 & 0.645 & 0.206 & 0.206 & 0.210 & 0.206 & 0.206 &       & 0.309 & 0.310 & 0.206 &       \\
        12h   & 0.189 & 0.197 & 0.207 & 0.206 & 906.205 & 40.452 & 4.025 & 0.186 & 0.195 & 0.671 & 0.198 & 0.186 & 0.421 & 0.417 & 0.306 & 0.186 \\
        \bottomrule
    \end{tabular}%
    }
\end{table*}

\paragraph{Irregular-grid generalisation.}
We next subsample each \(24\)-hour window to \(12\) irregular observation times drawn from a Gamma renewal process. The shape parameter \(k\) controls regularity: \(k=1\) produces highly irregular grids, while larger values approach uniform spacing. Details of the construction are given in Appendix~\ref{app:construction_irregular_sampling_grid}. We train with
\(
    k_{\mathrm{train}}\in\{1,10,100\}
\)
and evaluate either on matching irregular grids, \(k_{\mathrm{test}}=k_{\mathrm{train}}\), or on the regular grid, denoted \(k_{\mathrm{test}}\to\infty\).

Table~\ref{tab:ett15m_irregular_selected_allk_transposed_multirow_split_k1k10k100_hlines} reports CRPS and NRMSE over five random seeds. G-SLiCE is stable across both irregular and regular evaluation grids: its CRPS remains between \(0.13\) and \(0.15\), and its NRMSE between \(0.27\) and \(0.33\). TSFlow obtains comparable CRPS in some cases, but its NRMSE is consistently much larger, indicating unstable predictive means or large outliers. For example, at \(k_{\mathrm{train}}=100\) on matching irregular grids, TSFlow has CRPS \(0.21\) and NRMSE \(1.72\), while G-SLiCE has CRPS \(0.13\) and NRMSE \(0.27\). Figure~\ref{fig:example_forecasts_subsampling_small} shows representative forecasts for \(k_{\mathrm{train}}=1\).
Additional values of \(k\) are reported in Appendix~\ref{app:additional_results_irregular_sampling}.

\begin{wrapfigure}[16]{r}{0.4\textwidth}
    \centering
    \vspace{-1.41em}
    \includegraphics[width=\linewidth]{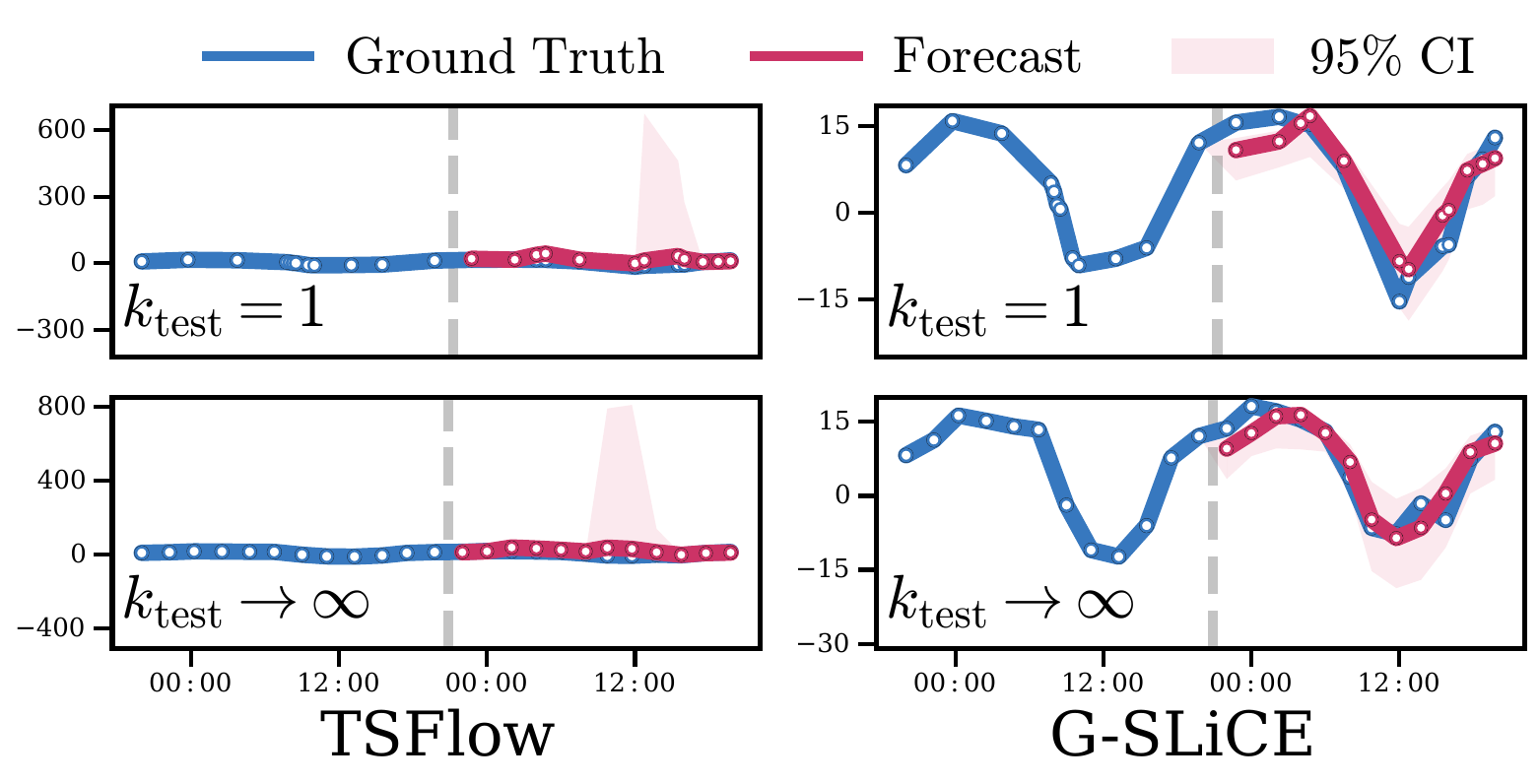}
    \caption{
    Example forecasts for TSFlow (left) and G-SLiCEs (right) trained with \(k_{\mathrm{train}}=1\) and evaluated for \(k_{\mathrm{test}}=1\) (top) or \(k_{\mathrm{test}}\to\infty\) (bottom). Curves show ground truth and predictive means; shaded regions indicate \(95\%\) confidence intervals.
    }
    \label{fig:example_forecasts_subsampling_small}
    \vspace{-1em}
\end{wrapfigure}

  \begin{table}[t!]
    \centering
    \caption{Comparison of G-SLiCE and TSFlow on the ETT 15-minute dataset from GluonTS with context and
prediction length of \(24\) hours. Each sequence is subsampled to \(12\) random time points drawn from a Gamma
distribution with shape parameter \(k\); larger \(k\) corresponds to a more regular grid. Columns indicate the training
irregularity \(k_{\mathrm{train}}\). For each \(k_{\mathrm{train}}\), we report CRPS and NRMSE separately as mean
\(\pm\) std. dev. over five random seeds. The first row per model shows i.i.d. irregular evaluation \(k_{\mathrm{test}}
= k_{\mathrm{train}}\), and the second row shows evaluation on the regular grid, i.e. \(k_{\mathrm{test}} \to \infty\).}
    \label{tab:ett15m_irregular_selected_allk_transposed_multirow_split_k1k10k100_hlines}

    \vspace{0.3em}
    \setlength{\tabcolsep}{3pt}
    \renewcommand{\arraystretch}{0.95}

    \resizebox{0.86\textwidth}{!}{%
    \begin{tabular}{llcc|cc|cc}
        \toprule
        & & \multicolumn{2}{c}{\(k_{\mathrm{train}}=1\)}
        & \multicolumn{2}{c}{\(k_{\mathrm{train}}=10\)}
        & \multicolumn{2}{c}{\(k_{\mathrm{train}}=100\)} \\
        \cmidrule(lr){3-4}\cmidrule(lr){5-6}\cmidrule(lr){7-8}
        & & CRPS & NRMSE & CRPS & NRMSE & CRPS & NRMSE \\
        \midrule
        \multirow{2}{*}{G-SLiCE}
        & \(k_{\mathrm{test}} = k_{\mathrm{train}}\)
        & \(0.13 \pm 0.01\) & \(0.28 \pm 0.05\)
        & \(0.13 \pm 0.01\) & \(0.32 \pm 0.08\)
        & \(0.13 \pm 0.02\) & \(0.27 \pm 0.05\) \\
        & \(k_{\mathrm{test}} \to \infty\)
        & \(0.15 \pm 0.03\) & \(0.33 \pm 0.08\)
        & \(0.14 \pm 0.01\) & \(0.27 \pm 0.08\)
        & \(0.13 \pm 0.02\) & \(0.27 \pm 0.04\) \\
        \midrule
        \multirow{2}{*}{TSFlow}
        & \(k_{\mathrm{test}} = k_{\mathrm{train}}\)
        & \(0.20 \pm 0.07\) & \(1.01 \pm 0.37\)
        & \(0.13 \pm 0.03\) & \(0.88 \pm 0.65\)
        & \(0.21 \pm 0.06\) & \(1.72 \pm 0.85\) \\
        & \(k_{\mathrm{test}} \to \infty\)
        & \(0.20 \pm 0.07\) & \(1.24 \pm 0.58\)
        & \(0.17 \pm 0.06\) & \(1.31 \pm 0.71\)
        & \(0.22 \pm 0.06\) & \(1.75 \pm 0.87\) \\
        \bottomrule
    \end{tabular}%
    }
\end{table}
\section{Conclusion}
\label{sec:conclusion}
This paper introduced G-SLiCE, a path-space flow-matching model with maximally expressive Structured Linear CDE backbone. Theoretically, we showed that pathwise universality implies universality of the induced pushforward laws, connecting deterministic expressivity to distributional time-series generation. On a concrete hard-core sequence example, we showed that the transition structure matters: G-SLiCE can uniformly approximate the induced distribution, whereas dense non-selective and diagonal selective exact-flow SSMs, which abstract S4D- and Mamba-style transitions, provably cannot. 

Empirically, G-SLiCE is competitive with strong statistical, neural, diffusion-based, and flow-based baselines on probabilistic forecasting and unconditional time-series generation. Its continuous-time formulation also improves robustness under changes in sampling frequency and irregular observation grids: G-SLiCE can be evaluated directly on the requested physical-time grid, while fixed-grid SSM backbones require auxiliary resampling rules and often become unstable.

\subsection{Limitations and future work}
\label{ssec:limitations}

One avenue for future work is improving the runtime of the SLiCE backbone, which would directly speed up our method. Our current implementation uses a first-order approximation of the exponential in Equation~\ref{eqn:recurrence_exp}. Although this supports parallel-in-time computation and parallel scan evaluation, efficient GPU kernels for structured matrix exponentials remain an important direction. Log-ODE-style approximations are another possibility, but are not currently compatible with our stacked architecture.
Second, our theory does not fully characterise the expressivity of the flow-matching dynamics used in practice. ODE-based flows impose structural constraints such as invertibility, which some distributional maps lack. Although we show that G-SLiCEs can be implemented as augmented flows, a sharper theory of which path-space distributional maps are reachable by the practical training is desirable.
Finally, Neural CDEs have been extended beyond Euclidean domains, including to graph-valued paths~\citep{qin2025learning, berndt2025permutation}. Extending G-SLiCEs to structured domains is a natural next step.

\clearpage

\section*{Acknowledgements}

We thank Marcel Kollovieh for engaging and insightful discussions and assistance with the TSFlow codebase.

This study received funding from the Klaus Tschira Stiftung gGmbH (HITS Lab). Benjamin Walker is supported by UK Research and Innovation (UKRI) through the Engineering and Physical Sciences Research Council (EPSRC) via Programme Grant [Grant No.\ UKRI1010: High order mathematical and computational infrastructure for streamed data that enhance contemporary generative and large language models] and CIMDA@Oxford, part of the AIR@InnoHK initiative funded by the Innovation and Technology Commission, HKSAR Government.

\bibliographystyle{plainnat}
\bibliography{main}

@article{albergo2023stochastic,
 author = {Albergo, Michael S. and Goldstein, Mark and Boffi, Nicholas M. and Ranganath, Rajesh and Vanden-Eijnden, Eric},
 journal = {arXiv preprint arXiv:2310.03725},
 title = {Stochastic interpolants with data-dependent couplings},
 year = {2023}
}

@article{alcaraz2022diffusion,
 author = {Alcaraz, Juan Miguel Lopez and Strodthoff, Nils},
 journal = {arXiv preprint arXiv:2208.09399},
 title = {Diffusion-based time series imputation and forecasting with structured state space models},
 year = {2022}
}

@article{aldaz2009bernstein,
 author = {Aldaz, J. M. and Kounchev, O. and Render, H.},
 journal = {Constructive Approximation},
 pages = {345--367},
 title = {Bernstein Operators for Exponential Polynomials},
 volume = {29},
 year = {2009}
}

@book{billingsley1999convergence,
 author = {Billingsley, Patrick},
 edition = {2},
 publisher = {Wiley},
 title = {Convergence of Probability Measures},
 year = {1999}
}

@inproceedings{bilovs2023modeling,
 author = {Bilo{\v{s}}, Marin and Rasul, Kashif and Schneider, Anderson and Nevmyvaka, Yuriy and G{\"u}nnemann, Stephan},
 booktitle = {International Conference on Machine Learning},
 organization = {PMLR},
 pages = {2452--2470},
 title = {Modeling temporal data as continuous functions with stochastic process diffusion},
 year = {2023}
}

@book{bogachev2007measure,
 address = {Berlin},
 author = {Bogachev, Vladimir I.},
 doi = {10.1007/978-3-540-34514-5},
 publisher = {Springer},
 title = {Measure Theory},
 year = {2007}
}

@article{cheridito2003fractional,
 author = {Cheridito, Patrick and Kawaguchi, Hiroshi and Maejima, Makoto},
 journal = {Electronic Journal of Probability},
 title = {Fractional Ornstein--Uhlenbeck Processes},
 volume = {8},
 year = {2003}
}

@inproceedings{cirone2024deepSSM,
 author = {Nicola Muca Cirone and Antonio Orvieto and Benjamin Walker and Cristopher Salvi and Terry Lyons},
 booktitle = {Proceedings of the 38th Conference on Neural Information Processing Systems (NeurIPS)},
 title = {Theoretical Foundations of Deep Selective State-Space Models},
 year = {2024}
}

@article{cybenko1989approximation,
 author = {Cybenko, George},
 journal = {Mathematics of Control, Signals and Systems},
 pages = {303--314},
 title = {Approximation by superpositions of a sigmoidal function},
 volume = {2},
 year = {1989}
}

@article{desai2021timevae,
 author = {Desai, Abhyuday and Freeman, Cynthia and Wang, Zuhui and Beaver, Ian},
 journal = {arXiv preprint arXiv:2111.08095},
 title = {Timevae: A variational auto-encoder for multivariate time series generation},
 year = {2021}
}

@misc{dheeru2017uci,
 author = {Dheeru, Dua and Taniskidou, E. Karra},
 title = {Uci machine learning repository},
 year = {2017}
}

@article{doob1942brownian,
 author = {Doob, Joseph L.},
 journal = {Annals of Mathematics},
 number = {2},
 title = {The Brownian Movement and Stochastic Equations},
 volume = {43},
 year = {1942}
}

@inproceedings{dua2017uci,
 author = {Dua, Dheeru and Graff, Casey and others},
 publisher = {Beijing},
 title = {UCI machine learning repository},
 year = {2017}
}

@misc{FiveThirtyEight2016,
 author = {FiveThirtyEight},
 howpublished = {\url{https://github.com/fivethirtyeight/uber-tlc-foil-response}},
 title = {Uber tlc foil response},
 year = {2016}
}

@article{flamary2021pot,
 author = {R{\'e}mi Flamary and Nicolas Courty and Alexandre Gramfort and Mokhtar Z. Alaya and Aur{\'e}lie Boisbunon and Stanislas Chambon and Laetitia Chapel and Adrien Corenflos and Kilian Fatras and Nemo Fournier and L{\'e}o Gautheron and Nathalie T.H. Gayraud and Hicham Janati and Alain Rakotomamonjy and Ievgen Redko and Antoine Rolet and Antony Schutz and Vivien Seguy and Danica J. Sutherland and Romain Tavenard and Alexander Tong and Titouan Vayer},
 journal = {Journal of Machine Learning Research},
 number = {78},
 pages = {1-8},
 title = {POT: Python Optimal Transport},
 url = {http://jmlr.org/papers/v22/20-451.html},
 volume = {22},
 year = {2021}
}

@misc{flamary2024pot,
 author = {Flamary, R{\'e}mi and Vincent-Cuaz, C{\'e}dric and Courty, Nicolas and Gramfort, Alexandre and Kachaiev, Oleksii and Quang Tran, Huy and David, Laurène and Bonet, Cl{\'e}ment and Cassereau, Nathan and Gnassounou, Th{\'e}o and Tanguy, Eloi and Delon, Julie and Collas, Antoine and Mazelet, Sonia and Chapel, Laetitia and Kerdoncuff, Tanguy and Yu, Xizheng and Feickert, Matthew and Krzakala, Paul and Liu, Tianlin and Fernandes Montesuma, Eduardo},
 title = {POT Python Optimal Transport (version 0.9.5)},
 url = {https://github.com/PythonOT/POT},
 year = {2024}
}

@book{folland1999realanalysis,
 author = {Folland, Gerald B.},
 edition = {2},
 publisher = {Wiley},
 title = {Real Analysis: Modern Techniques and Their Applications},
 year = {1999}
}

@inproceedings{gasthaus2019probabilistic,
 author = {Gasthaus, Jan and Benidis, Konstantinos and Wang, Yuyang and Rangapuram, Syama Sundar and Salinas, David and Flunkert, Valentin and Januschowski, Tim},
 booktitle = {The 22nd international conference on artificial intelligence and statistics},
 organization = {PMLR},
 pages = {1901--1910},
 title = {Probabilistic forecasting with spline quantile function rnns},
 year = {2019}
}

@article{gluonts_arxiv,
 author = {Alexandrov, A. and Benidis, K. and Bohlke-Schneider, M. and
Flunkert, V. and Gasthaus, J. and Januschowski, T. and Maddix, D. C.
and Rangapuram, S. and Salinas, D. and Schulz, J. and Stella, L. and
Türkmen, A. C. and Wang, Y.},
 journal = {arXiv preprint arXiv:1906.05264},
 title = {{GluonTS: Probabilistic Time Series Modeling in Python}},
 year = {2019}
}

@article{gluonts_jmlr,
 author = {Alexandrov, Alexander and Benidis, Konstantinos and Bohlke-Schneider, Michael and Flunkert, Valentin and Gasthaus, Jan and Januschowski, Tim and Maddix, Danielle C. and Rangapuram, Syama and Salinas, David and Schulz, Jasper and Stella, Lorenzo and T{\"u}rkmen, Ali Caner and Wang, Yuyang},
 journal = {Journal of Machine Learning Research},
 number = {116},
 pages = {1--6},
 title = {{GluonTS: Probabilistic and Neural Time Series Modeling in Python}},
 url = {http://jmlr.org/papers/v21/19-820.html},
 volume = {21},
 year = {2020}
}

@article{gneiting2007strictly,
 author = {Gneiting, Tilmann and Raftery, Adrian E.},
 journal = {Journal of the American statistical Association},
 number = {477},
 pages = {359--378},
 title = {Strictly proper scoring rules, prediction, and estimation},
 volume = {102},
 year = {2007}
}

@inproceedings{godahewa2021monash,
 author = {Godahewa, Rakshitha and Bergmeir, Christoph and Webb, Geoffrey I. and Hyndman, Rob J. and Montero-Manso, Pablo},
 booktitle = {Neural Information Processing Systems Track on Datasets and Benchmarks},
 title = {Monash time series forecasting archive},
 year = {2021}
}

@inproceedings{godahewa2monash,
 author = {Godahewa, Rakshitha Wathsadini and Bergmeir, Christoph and Webb, Geoffrey I and Hyndman, Rob and Montero-Manso, Pablo},
 booktitle = {Thirty-fifth Conference on Neural Information Processing Systems Datasets and Benchmarks Track (Round 2)},
 title = {Monash Time Series Forecasting Archive},
 year = {2021}
}

@inproceedings{haoyietal-informer-2021,
 author = {Haoyi Zhou and
Shanghang Zhang and
Jieqi Peng and
Shuai Zhang and
Jianxin Li and
Hui Xiong and
Wancai Zhang},
 booktitle = {The Thirty-Fifth {AAAI} Conference on Artificial Intelligence, {AAAI} 2021, Virtual Conference},
 number = {12},
 pages = {11106--11115},
 publisher = {{AAAI} Press},
 title = {Informer: Beyond Efficient Transformer for Long Sequence Time-Series Forecasting},
 volume = {35},
 year = {2021}
}

@article{hornik1991approximation,
 author = {Hornik, Kurt},
 journal = {Neural Networks},
 number = {2},
 pages = {251--257},
 publisher = {Elsevier},
 title = {Approximation capabilities of multilayer feedforward networks},
 volume = {4},
 year = {1991}
}

@book{hyndman2008forecasting,
 author = {Hyndman, Rob and Koehler, Anne B. and Ord, J. Keith and Snyder, Ralph D.},
 publisher = {Springer Science \& Business Media},
 title = {Forecasting with exponential smoothing: the state space approach},
 year = {2008}
}

@article{kerrigan2023functional,
 author = {Kerrigan, Gavin and Migliorini, Giosue and Smyth, Padhraic},
 journal = {arXiv preprint arXiv:2305.17209},
 title = {Functional flow matching},
 year = {2023}
}

@article{kollovieh2023predict,
 author = {Kollovieh, Marcel and Ansari, Abdul Fatir and Bohlke-Schneider, Michael and Zschiegner, Jasper and Wang, Hao and Wang, Yuyang Bernie},
 journal = {Advances in Neural Information Processing Systems},
 title = {Predict, refine, synthesize: Self-guiding diffusion models for probabilistic time series forecasting},
 volume = {36},
 year = {2023}
}

@article{kollovieh2024flow,
 author = {Kollovieh, Marcel and Lienen, Marten and L{\"u}dke, David and Schwinn, Leo and G{\"u}nnemann, Stephan},
 github = {https://github.com/marcelkollovieh/TSFlow},
 journal = {The Thirteenth International Conference on Learning Representations},
 openreview = {https://openreview.net/forum?id=uxVBbSlKQ4&},
 shortjournal = {ICLR},
 title = {Flow Matching with Gaussian Process Priors for Probabilistic Time Series Forecasting},
 year = {2025}
}

@inproceedings{lai2018modeling,
 author = {Lai, Guokun and Chang, Wei-Cheng and Yang, Yiming and Liu, Hanxiao},
 booktitle = {The 41st international ACM SIGIR conference on research \& development in information retrieval},
 pages = {95--104},
 title = {Modeling long-and short-term temporal patterns with deep neural networks},
 year = {2018}
}

@article{lim2021temporal,
 author = {Lim, Bryan and Ar{\i}k, Sercan {\"O}. and Loeff, Nicolas and Pfister, Tomas},
 journal = {International Journal of Forecasting},
 number = {4},
 pages = {1748--1764},
 title = {Temporal fusion transformers for interpretable multi-horizon time series forecasting},
 volume = {37},
 year = {2021}
}

@article{lopezalcaraz2022diffusionbased,
 author = {Juan Lopez Alcaraz and Nils Strodthoff},
 issn = {2835-8856},
 journal = {Transactions on Machine Learning Research},
 title = {Diffusion-based Time Series Imputation and Forecasting with Structured State Space Models},
 url = {https://openreview.net/forum?id=hHiIbk7ApW},
 year = {2022}
}

@article{makridakis2020m4,
 author = {Makridakis, Spyros and Spiliotis, Evangelos and Assimakopoulos, Vassilios},
 journal = {International Journal of Forecasting},
 number = {1},
 pages = {54--74},
 title = {The m4 competition: 100,000 time series and 61 forecasting methods},
 volume = {36},
 year = {2020}
}

@inproceedings{movahedi2025fixedpointrnnsdiagonaldense,
 author = {Sajad Movahedi and Felix Sarnthein and Nicola Muca Cirone and Antonio Orvieto},
 booktitle = {Proceedings of the 39th Conference on Neural Information Processing Systems (NeurIPS)},
 title = {Fixed-Point RNNs: From Diagonal to Dense in a Few Iterations},
 year = {2025}
}

@article{oord2016wavenet,
 author = {Oord, Aaron van den and Dieleman, Sander and Zen, Heiga and Simonyan, Karen and Vinyals, Oriol and Graves, Alex and Kalchbrenner, Nal and Senior, Andrew and Kavukcuoglu, Koray},
 journal = {arXiv preprint arXiv:1609.03499},
 title = {Wavenet: A generative model for raw audio},
 year = {2016}
}

@inproceedings{orvieto2024universality,
 author = {Orvieto, Antonio and De, Soham and Gulcehre, Caglar and Pascanu, Razvan and Smith, Samuel L.},
 booktitle = {Proceedings of the 41st International Conference on Machine Learning},
 title = {Universality of Linear Recurrences Followed by Non-linear Projections: Finite-Width Guarantees and Benefits of Complex Eigenvalues},
 year = {2024}
}

@book{Rasmussen2006Gaussian,
 author = {Rasmussen, Carl Edward and Williams, Christopher K. I.},
 publisher = {The MIT Press},
 title = {Gaussian Processes for Machine Learning},
 year = {2006}
}

@inproceedings{rasul2021autoregressive,
 author = {Rasul, Kashif and Seward, Calvin and Schuster, Ingmar and Vollgraf, Roland},
 booktitle = {International Conference on Machine Learning},
 organization = {PMLR},
 pages = {8857--8868},
 title = {Autoregressive denoising diffusion models for multivariate probabilistic time series forecasting},
 year = {2021}
}

@article{salinas2020deepar,
 author = {Salinas, David and Flunkert, Valentin and Gasthaus, Jan and Januschowski, Tim},
 journal = {International journal of forecasting},
 number = {3},
 pages = {1181--1191},
 title = {Deepar: Probabilistic forecasting with autoregressive recurrent networks},
 volume = {36},
 year = {2020}
}

@article{scikit-learn,
 author = {Pedregosa, F. and Varoquaux, G. and Gramfort, A. and Michel, V.
and Thirion, B. and Grisel, O. and Blondel, M. and Prettenhofer, P.
and Weiss, R. and Dubourg, V. and Vanderplas, J. and Passos, A. and
Cournapeau, D. and Brucher, M. and Perrot, M. and Duchesnay, E.},
 journal = {Journal of Machine Learning Research},
 pages = {2825--2830},
 title = {Scikit-learn: Machine Learning in {P}ython},
 volume = {12},
 year = {2011}
}

@inproceedings{sohl2015deep,
 author = {Sohl-Dickstein, Jascha and Weiss, Eric and Maheswaranathan, Niru and Ganguli, Surya},
 booktitle = {International conference on machine learning},
 organization = {PMLR},
 pages = {2256--2265},
 title = {Deep unsupervised learning using nonequilibrium thermodynamics},
 year = {2015}
}

@article{tashiro2021csdi,
 author = {Tashiro, Yusuke and Song, Jiaming and Song, Yang and Ermon, Stefano},
 journal = {Advances in Neural Information Processing Systems},
 pages = {24804--24816},
 title = {Csdi: Conditional score-based diffusion models for probabilistic time series imputation},
 volume = {34},
 year = {2021}
}

@inproceedings{tong2023improving,
 author = {Tong, Alexander and Malkin, Nikolay and Huguet, Guillaume and Zhang, Yanlei and Rector-Brooks, Jarrid and FATRAS, Kilian and Wolf, Guy and Bengio, Yoshua},
 booktitle = {ICML Workshop on New Frontiers in Learning, Control, and Dynamical Systems},
 title = {Improving and generalizing flow-based generative models with minibatch optimal transport},
 year = {2023}
}

@book{villani2009optimaltransport,
 address = {Berlin, Heidelberg},
 author = {Villani, C{\'e}dric},
 doi = {10.1007/978-3-540-71050-9},
 publisher = {Springer},
 title = {Optimal Transport: Old and New},
 year = {2009}
}

@inproceedings{walker2025structuredlinearcdesmaximally,
 author = {Benjamin Walker and Lingyi Yang and Nicola Muca Cirone and Cristopher Salvi and Terry Lyons},
 booktitle = {Proceedings of the 39th Conference on Neural Information Processing Systems (NeurIPS)},
 title = {Structured Linear CDEs: Maximally Expressive and Parallel-in-Time Sequence Models},
 year = {2025}
}

@inproceedings{wang2023state,
 author = {Wang, Shida and Xue, Beichen},
 booktitle = {Proceedings of the 37th Conference on Neural Information Processing Systems (NeurIPS)},
 title = {State-space Models with Layer-wise Nonlinearity are Universal Approximators with Exponential Decaying Memory},
 year = {2023}
}

@inproceedings{Yuqietal-2023-PatchTST,
 author = {Nie, Yuqi and
H. Nguyen, Nam and
Sinthong, Phanwadee and 
Kalagnanam, Jayant},
 booktitle = {International Conference on Learning Representations},
 title = {A Time Series is Worth 64 Words: Long-term Forecasting with Transformers},
 year = {2023}
}

@inproceedings{zeng2023transformers,
 author = {Zeng, Ailing and Chen, Muxi and Zhang, Lei and Xu, Qiang},
 booktitle = {Proceedings of the AAAI conference on artificial intelligence},
 pages = {11121--11128},
 title = {Are transformers effective for time series forecasting?},
 volume = {37},
 year = {2023}
}

@inproceedings{lipman2023flow,
title={Flow Matching for Generative Modeling},
author={Yaron Lipman and Ricky T. Q. Chen and Heli Ben-Hamu and Maximilian Nickel and Matt Le},
booktitle = {The Eleventh International Conference on Learning Representations},
year = {2023}
}

@inproceedings{berndt2025permutation,
  title={Permutation Equivariant Neural Controlled Differential Equations for Dynamic Graph Representation Learning},
  author={Berndt, Torben and Walker, Benjamin and Qin, Tiexin and St{\"u}hmer, Jan and Kormilitzin, Andrey},
  booktitle={Proceedings of the 39th Conference on Neural Information Processing Systems (NeurIPS)},
  year={2025}
}

@article{qin2025learning,
  title={Learning dynamic graph embeddings with neural controlled differential equations},
  author={Qin, Tiexin and Walker, Benjamin and Lyons, Terry and Yan, Hong and Li, Haoliang},
  journal={IEEE Transactions on Pattern Analysis and Machine Intelligence},
  year={2025},
  publisher={IEEE}
}

@inproceedings{
chen2024flow,
title={Flow Matching on General Geometries},
author={Ricky T. Q. Chen and Yaron Lipman},
booktitle={The Twelfth International Conference on Learning Representations},
year={2024},
}

@inproceedings{
gat2024discrete,
title={Discrete Flow Matching},
author={Itai Gat and Tal Remez and Neta Shaul and Felix Kreuk and Ricky T. Q. Chen and Gabriel Synnaeve and Yossi Adi and Yaron Lipman},
booktitle={The Thirty-eighth Annual Conference on Neural Information Processing Systems},
year={2024},
}

@article{lim2024parallelizing,
  title={Parallelizing nonlinear sequential models over the sequence length},
  author={Lim, Yi Heng and Zhu, Qi and Selfridge, Joshua and Kasim, Muhammad Firmansyah},
  journal={International Conference on Learning Representations},
  year={2024}
}

@inproceedings{gonzalez2024towards,
  title={Towards scalable and stable parallelization of nonlinear RNNs},
  author={Gonzalez, Xavier and Warrington, Andrew and Smith, Jimmy T. H. and Linderman, Scott W.},
  booktitle={Advances in Neural Information Processing Systems},
  year={2024}
}

@inproceedings{fan2024advancing,
  title={Advancing regular language reasoning in linear recurrent neural networks},
  author={Fan, Ting-Han and Chi, Ta-Chung and Rudnicky, Alexander},
  booktitle={Proceedings of the 2024 Conference of the North American Chapter of the Association for Computational Linguistics: Human Language Technologies, Volume 2: Short Papers},
  pages={45--53},
  year={2024},
  address={Mexico City, Mexico},
  publisher={Association for Computational Linguistics}
}

@inproceedings{yang2024parallelizing,
  title={Parallelizing linear transformers with the delta rule over sequence length},
  author={Yang, Songlin and Wang, Bailin and Zhang, Yu and Shen, Yikang and Kim, Yoon},
  booktitle={The Thirty-eighth Annual Conference on Neural Information Processing Systems},
  year={2024}
}

@article{yang2024improving,
  title={Gated delta networks: Improving mamba2 with delta rule.},
  author={Yang, Songlin and Kautz, Jan and Hatamizadeh, Ali},
  journal={arXiv preprint arXiv:2412.06464},
  year={2024}
}

@article{siems2025deltaproduct,
  title={DeltaProduct: Imprpoving state-tracking in linear RNNs via Householder products},
  author={Siems, Julien and Carstensen, Timur and Zela, Arber and Hutter, Frank and Pontil, Massimiliano and Grazzi, Riccardo},
  journal={arXiv preprint arXiv:2502.10297},
  year={2025}
}

@article{gu2023mamba,
  title={Mamba: Linear-time sequence modeling with selective state spaces},
  author={Gu, Albert and Dao, Tri},
  journal={arXiv preprint arXiv:2312.00752},
  year={2023}
}

@inproceedings{dao2024transformers,
  title={Transformers are SSMs: Generalized models and efficient algorithms through structured state space duality},
  author={Dao, Tri and Gu, Albert},
  booktitle={Proceedings of the 41st International Conference on Machine Learning},
  year={2024},
  publisher={JMLR.org}
}

@article{peng2025rwkv,
  title={RWKV-7 ``Goose'' with expressive dynamic state evolution},
  author={Peng, Bo and Zhang, Ruichong and Goldstein, Daniel and Alcaide, Eric and Du, Xingjian and Hou, Haowen and Lin, Jiaju and Liu, Jiaxing and Lu, Janna and Merrill, William and Song, Guangyu and Tan, Kaifeng and Utpala, Saiteja and Wilce, Nathan and Wind, Johan S. and Wu, Tianyi and Wuttke, Daniel and Zhou-Zheng, Christian},
  journal={arXiv preprint arXiv:2503.14456},
  year={2025}
}

@article{qin2024hgrn2,
  title={HGRN2: Gated linear RNNs with state expansion},
  author={Qin, Zhen and Yang, Songlin and Sun, Weixuan and Shen, Xuyang and Li, Dong and Sun, Weigao and Zhong, Yiran},
  journal={arXiv preprint arXiv:2404.07904},
  year={2024}
}

@inproceedings{beck2024xlstm,
  title={xLSTM: Extended long short-term memory},
  author={Beck, Maximilian and P{\"o}ppel, Korbinian and Spanring, Markus and Auer, Andreas and Prudnikova, Oleksandra and Kopp, Michael and Klambauer, G{\"u}nter and Brandstetter, Johannes and Hochreiter, Sepp},
  booktitle={Proceedings of the 38th Conference on Neural Information Processing Systems},
  year={2024}
}

@article{yang2024gated,
  title={Gated linear attention transformers with hardware-efficient training},
  author={Yang, Songlin and Wang, Bailin and Shen, Yikang and Panda, Rameswar and Kim, Yoon},
  journal={arXiv preprint arXiv:2312.06635},
  year={2024}
}

@inproceedings{peng2021random,
  title={Random feature attention},
  author={Peng, Hao and Pappas, Nikolaos and Yogatama, Dani and Schwartz, Roy and Smith, Noah A. and Kong, Lingpeng},
  booktitle={International Conference on Learning Representations},
  year={2021}
}

@inproceedings{zhang2024gated,
  title={Gated slot attention for efficient linear-time sequence modeling},
  author={Zhang, Yuxuan and Yang, Shiliang and Zhu, Rong and Zhang, Yichong and Cui, Lei and Wang, Yongjing and Wang, Bin and Shi, Feng and Wang, Bing and Bi, Wei and Zhou, Ping and Fu, Guoxin},
  booktitle={Proceedings of the Thirty-eighth Annual Conference on Neural Information Processing Systems},
  year={2024}
}

@article{sun2025learning,
  title={Learning to (learn at test time): RNNs with expressive hidden states},
  author={Sun, Yu and Li, Xinhao and Dalal, Karan and Xu, Jiarui and Vikram, Arjun and Zhang, Genghan and Dubois, Yann and Chen, Xinlei and Wang, Xiaolong and Koyejo, Sanmi and Hashimoto, Tatsunori and Guestrin, Carlos},
  journal={arXiv preprint arXiv:2407.04620},
  year={2025}
}

@article{behrouz2024titans,
  title={Titans: Learning to memorize at test time},
  author={Behrouz, Ali and Zhong, Peilin and Mirrokni, Vahab},
  journal={arXiv preprint arXiv:2501.00663},
  year={2024}
}

@article{price2025probabilistic,
  title={Probabilistic weather forecasting with machine learning},
  author={Price, Ilan and Sanchez-Gonzalez, Alvaro and Alet, Ferran and Andersson, Tom R and El-Kadi, Andrew and Masters, Dominic and Ewalds, Timo and Stott, Jacklynn and Mohamed, Shakir and Battaglia, Peter and others},
  journal={Nature},
  volume={637},
  number={8044},
  pages={84--90},
  year={2025},
  publisher={Nature Publishing Group UK London}
}

@article{taieb2015probabilistic,
  title={Probabilistic time series forecasting with boosted additive models: an application to smart meter data},
  author={Taieb, Souhaib Ben and Huser, Raphael and Hyndman, Rob J and Genton, Marc G and others},
  journal={Department of Economics and business statistics, Monash University},
  year={2015}
}

@inproceedings{qian2023uncertainty,
  title={Uncertainty quantification for traffic forecasting: A unified approach},
  author={Qian, Weizhu and Zhang, Dalin and Zhao, Yan and Zheng, Kai and Yu, James JQ},
  booktitle={2023 IEEE 39th International Conference on Data Engineering (ICDE)},
  pages={992--1004},
  year={2023},
  organization={IEEE}
}


\newpage
\appendixpage
\DoToC
\appendix
\section{Universal Time Series Generation}\label{app:theory}

This section makes precise the sense in which Linear NCDEs are universal models for time series generation: they can approximate any target path law obtained by applying a continuous causal transformation to a latent path law. Under compact support, this approximation holds with exact $W_\infty$ control. For tight non-compact latent laws, the same argument yields high-probability approximation on compact sets carrying arbitrarily large mass. First, we show that in a general setting, maximal expressivity at the level of deterministic functions implies approximation of the corresponding pushforward distributions in $W_\infty$ whenever the input law is concentrated on a compact set. We then apply this observation to path-to-path models, using recent universality results for Linear NCDEs with non-linear readouts to obtain the corresponding distributional statement on path space. Finally, we contrast this with S4 and Mamba, and discuss how these ideas apply in the conditional Gaussian process setting used in practice, where interpolation and non-compact latent laws introduce additional technical considerations.

\subsection{From Maximal Expressivity to Distributional Approximation}

This subsection explains how a deterministic approximation statement can be converted into a distributional approximation result. The starting point is maximal expressivity, which is the ability to approximate continuous functions uniformly on compact sets. The conclusion is that, when the input law is concentrated on a compact set and the target map is continuous on that set, the pushforward distributions induced by the model can approximate the target pushforward distribution in $W_\infty$. The argument is elementary, but it is useful to state it explicitly because it clarifies how approximation properties of functions transfer to approximation properties of the corresponding pushforward measures.

\begin{definition}[Maximal expressivity~\citep{walker2025structuredlinearcdesmaximally}]
\label{def:max_exp_metric_global}
Let $(\mathcal{X},\rho_\mathcal{X})$ and $(\mathcal{Y},\rho_\mathcal{Y})$ be metric spaces, and let $\mathcal{F}=\{f_\theta:\mathcal{X}\to\mathcal{Y}\mid \theta\in\Theta\}$ be a class of functions.
We say that $\mathcal{F}$ is maximally expressive, or universal, if for every compact set $\mathcal{K}\subset\mathcal{X}$ and every continuous function $f:\mathcal{K}\to\mathcal{Y}$,
\begin{equation}
\label{eq:max_exp_metric_def_global}
\forall \varepsilon>0,\;\exists \theta\in\Theta
\quad\text{s.t.}\quad
\sup_{x\in\mathcal{K}} \rho_\mathcal{Y}\big(f(x),f_\theta(x)\big)\leq\varepsilon.
\end{equation}
\end{definition}

When $\mathcal{X}=\mathbb{R}^d$ with $\rho_\mathcal{X}(x_1,x_2)=\|x_1-x_2\|_2$ and $\mathcal{Y}=\mathbb{R}$ with $\rho_\mathcal{Y}(y_1,y_2)=|y_1-y_2|$, Definition~\ref{def:max_exp_metric_global} is the standard requirement of uniform approximation on compact subsets of $\mathbb{R}^d$. In this setting, classical universal approximation theorems show that multi-layer perceptrons with suitable activation functions form a maximally expressive function class in the sense of Definition~\ref{def:max_exp_metric_global}~\citep{cybenko1989approximation, hornik1991approximation}. Our goal here is to show that this notion of uniform approximation also yields a natural approximation guarantee at the level of probability measures.

To state this precisely, we recall the basic measure-theoretic notions appearing in the argument. These definitions are standard, but we include them for completeness and to keep the appendix self-contained. We begin with the underlying measurable structure.

\begin{definition}[$\sigma$-algebra~\citep{bogachev2007measure}]
Let $\mathcal{X}$ be a set. A collection $\Sigma$ of subsets of $\mathcal{X}$ is called a $\sigma$-algebra on $\mathcal{X}$ if
\begin{enumerate}
    \item $\mathcal{X}\in\Sigma$,
    \item if $A\in\Sigma$ then $\mathcal{X}\setminus A\in\Sigma$,
    \item if $(A_n)_{n\ge 1}$ is a sequence of sets in $\Sigma$, then $\bigcup_{n\ge 1} A_n\in\Sigma$.
\end{enumerate}
The pair $(\mathcal{X},\Sigma)$ is called a measurable space.
\end{definition}

A $\sigma$-algebra specifies which subsets of $\mathcal{X}$ are measurable, and therefore which events can be assigned probabilities. Since $\mathcal{X}$ is assumed to be a metric space, there is a canonical choice of measurable structure, namely the Borel $\sigma$-algebra generated by the open sets.

\begin{definition}[Borel $\sigma$-algebra~\citep{bogachev2007measure}]
Let $(\mathcal{X},\rho_\mathcal{X})$ be a metric space. The Borel $\sigma$-algebra on $\mathcal{X}$, denoted $\mathcal{B}(\mathcal{X})$, is the smallest $\sigma$-algebra containing all open subsets of $\mathcal{X}$.
\end{definition}

Once the measurable structure has been fixed, we can speak of probability measures on $\mathcal{X}$.

\begin{definition}[Borel probability measure~\citep{bogachev2007measure}]
Let $(\mathcal{X},\rho_\mathcal{X})$ be a metric space. A Borel probability measure on $\mathcal{X}$ is a probability measure on the measurable space $(\mathcal{X},\mathcal{B}(\mathcal{X}))$.
\end{definition}

The next notion records where a measure is locally non-trivial.

\begin{definition}[Support~\citep{bogachev2007measure}]
Let $\mu$ be a Borel probability measure on a metric space $(\mathcal{X},\rho_\mathcal{X})$. The support of $\mu$ is
\begin{equation}
\label{eq:support_def}
\mathrm{supp}(\mu)
=\bigl\{x\in\mathcal{X}:\ \mu\bigl(B_{\rho_\mathcal{X}}(x,r)\bigr)>0\ \text{for all } r>0\bigr\}.
\end{equation}
\end{definition}

Given a measurable map $T:\mathcal{X}\to\mathcal{Y}$ and an input law $\mu$ on $\mathcal{X}$, the natural output law is the distribution obtained by transporting $\mu$ through $T$. This is captured by the pushforward measure.

\begin{definition}[Pushforward measure~\citep{bogachev2007measure}]
Let $\mu$ be a Borel probability measure on $\mathcal{X}$ and let $T:\mathcal{X}\to\mathcal{Y}$ be Borel measurable.
The pushforward of $\mu$ by $T$ is the Borel probability measure $T_\#\mu$ on $\mathcal{Y}$ defined by
\begin{equation}
\label{eq:pushforward_def}
(T_\#\mu)(A)=\mu\big(T^{-1}(A)\big), \qquad A\in\mathcal{B}(\mathcal{Y}).
\end{equation}
\end{definition}

Thus, if $X\sim\mu$, then $T_\#\mu$ is simply the law of the random variable $T(X)$. In our setting, $T$ will denote the target transformation and $f_\theta$ a model approximation to $T$. The question is whether closeness of $f_\theta$ to $T$ at the level of points implies closeness of $(f_\theta)_\#\mu$ to $T_\#\mu$ at the level of probability measures. 

To measure this closeness between probability measures, we use the $\infty$-Wasserstein distance on the output space $\mathcal{Y}$. This distance records the smallest possible essential worst-case transport cost over all couplings of the two measures. For the $W_\infty$ argument below, we additionally assume that the output space $(\mathcal{Y},\rho_\mathcal{Y})$ is separable, meaning that it contains a countable dense subset.

\begin{definition}[$W_\infty$~\citep{villani2009optimaltransport}]
Let $(\mathcal{Y},\rho_\mathcal{Y})$ be a separable metric space, and let $\mu,\nu$ be Borel probability measures on $\mathcal{Y}$.
A coupling of $\mu$ and $\nu$ is a Borel probability measure $\pi$ on $\mathcal{Y}\times\mathcal{Y}$ whose first marginal is $\mu$ and whose second marginal is $\nu$.
Write $\Pi(\mu,\nu)$ for the set of all such couplings.
The $\infty$-Wasserstein distance, possibly taking the value $+\infty$, is
\begin{equation}
\label{eq:winfty_def}
W_\infty(\mu,\nu)=\inf_{\pi\in\Pi(\mu,\nu)} \ \operatorname*{ess\,sup}_{(y,z)\sim\pi} \,\rho_\mathcal{Y}(y,z).
\end{equation}
\end{definition}

The relevance of $W_\infty$ here is that it interacts very naturally with uniform approximation. If two functions are uniformly close on a set of full $\mu$-measure, then applying them to the same input sample immediately produces a coupling whose transport cost is uniformly controlled. The next lemma formalises this simple observation.

\begin{lemma}[$W_\infty$ control via a pointwise transport bound]
\label{lem:winfty_pointwise_bound}
Let $(\mathcal{X},\rho_\mathcal{X})$ be a metric space, let $(\mathcal{Y},\rho_\mathcal{Y})$ be a separable metric space, and let $\mu$ be a Borel probability measure on $\mathcal{X}$.
Let $T:\mathcal{X}\to\mathcal{Y}$ and $f:\mathcal{X}\to\mathcal{Y}$ be Borel measurable.
Then
\begin{equation}
\label{eq:winfty_pointwise_bound}
W_\infty(f_\#\mu,T_\#\mu)\le \operatorname*{ess\,sup}_{x\sim\mu} \rho_\mathcal{Y}\big(f(x),T(x)\big).
\end{equation}
\end{lemma}

\begin{proof}
Since $\mathcal{Y}$ is separable, the map $x\mapsto (f(x),T(x))$ from $\mathcal{X}$ to $\mathcal{Y}\times\mathcal{Y}$ is Borel measurable.
Define
\begin{equation}
\pi = (f,T)_\#\mu .
\end{equation}
Then $\pi\in\Pi(f_\#\mu,T_\#\mu)$, and
\begin{equation}
\label{eq:winfty_pointwise_bound_coupling}
\operatorname*{ess\,sup}_{(y,z)\sim\pi}\, \rho_\mathcal{Y}(y,z)
= \operatorname*{ess\,sup}_{x\sim\mu}\, \rho_\mathcal{Y}\big(f(x),T(x)\big).
\end{equation}
Taking the infimum over all couplings in \eqref{eq:winfty_def} yields \eqref{eq:winfty_pointwise_bound}.
\end{proof}

The lemma shows that distributional approximation in $W_\infty$ follows immediately from an almost-sure pointwise bound. We now combine this observation with maximal expressivity. The key approximation-theoretic assumptions in the corollary are that the input law is concentrated on a compact set and that the target map is continuous on that set.

\begin{corollary}[Distributional approximation from maximal expressivity]
\label{cor:max_exp_distributional}
Let $(\mathcal{X},\rho_\mathcal{X})$ be a metric space, let $(\mathcal{Y},\rho_\mathcal{Y})$ be a separable metric space, and let $\mathcal{F}=\{f_\theta:\mathcal{X}\to\mathcal{Y}\mid \theta\in\Theta\}$, where each $f_\theta$ is Borel measurable, be maximally expressive in the sense of Definition~\ref{def:max_exp_metric_global}.
Let $\mu$ be a Borel probability measure on $\mathcal{X}$ such that $\mu(\mathcal{K})=1$ for some compact set $\mathcal{K}\subset\mathcal{X}$.
Let $T:\mathcal{X}\to\mathcal{Y}$ be Borel measurable and assume that $T|_{\mathcal{K}}:\mathcal{K}\to\mathcal{Y}$ is continuous.
Then for every $\varepsilon>0$ there exists $\theta\in\Theta$ such that
\begin{equation}
\label{eq:cor_max_exp_winfty}
W_\infty\big((f_\theta)_\#\mu,T_\#\mu\big)\le \varepsilon.
\end{equation}
\end{corollary}

\begin{proof}
Fix $\varepsilon>0$. By maximal expressivity applied to the compact set $\mathcal{K}$ and the continuous map $T|_{\mathcal{K}}:\mathcal{K}\to\mathcal{Y}$, choose $\theta\in\Theta$ such that
\begin{equation}
\label{eq:cor_max_exp_uniform}
\sup_{x\in\mathcal{K}} \rho_\mathcal{Y}\bigl(f_\theta(x),T(x)\bigr) \leq \varepsilon.
\end{equation}
Since $\mu(\mathcal{K})=1$, it follows that
\begin{equation}
\label{eq:cor_max_exp_esssup}
\operatorname*{ess\,sup}_{x\sim\mu} \rho_\mathcal{Y}\bigl(f_\theta(x),T(x)\bigr) \le \varepsilon.
\end{equation}
Applying Lemma~\ref{lem:winfty_pointwise_bound} with $f=f_\theta$ yields \eqref{eq:cor_max_exp_winfty}.
\end{proof}

Corollary~\ref{cor:max_exp_distributional} shows that maximal expressivity at the level of deterministic functions automatically yields a corresponding universality statement for pushforward measures generated from input laws that are concentrated on a compact set, provided the target map is Borel measurable on $\mathcal{X}$ and continuous on a compact full-measure set. In other words, once a model class can approximate continuous functions uniformly on compact sets, it can also approximate the distributional transformations induced by such functions on any input law $\mu$ satisfying $\mu(\mathcal{K})=1$ for some compact $\mathcal{K}\subset\mathcal{X}$.

This observation is useful in machine learning settings where one is interested not only in approximating a target map pointwise, but also in reproducing the distribution of outputs generated by that map over a population of inputs. The corollary shows that, under concentration on a compact full-measure set and continuity on that set, no separate distributional approximation theorem is needed: it follows directly from the standard uniform approximation property together with the elementary coupling argument of Lemma~\ref{lem:winfty_pointwise_bound}.

\subsection{Path-to-Path Models}

We now specialise to path space. Throughout this subsection, let $\mathcal{X}(d)$ denote the space
\begin{equation}
\mathcal{X}(d)=C^{1,0}([t_0,t_n],\mathbb{R}^d),
\end{equation}
consisting of absolutely continuous, time-augmented $d$-dimensional paths on the interval $[t_0,t_n]$ which all begin at the same point, endowed with the $1$-variation topology. For output paths, we use the supremum metric
\begin{equation}
\label{eq:path_sup_metric}
\rho_\infty(Y,\widetilde Y)=\sup_{t\in[t_0,t_n]}\|Y_t-\widetilde Y_t\|_2,
\qquad
Y,\widetilde Y\in C([t_0,t_n],\mathbb{R}^{d_y}).
\end{equation}

We next define a Linear NCDE with a generic readout.

\begin{definition}[Linear NCDE \citep{cirone2024deepSSM, walker2025structuredlinearcdesmaximally}]
\label{def:lin_NCDE}
Let $L^1_\theta\in\mathbb{R}^{d_h\times d_X}$, $b_\theta\in\mathbb{R}^{d_h}$, and $A_\theta\in\mathbb{R}^{d_h\times d_X\times d_h}$ be trainable parameters, and let
\begin{equation}
R_\theta:\mathbb{R}^{d_h}\to\mathbb{R}^{d_y}
\end{equation}
be a readout map.
Then the corresponding Linear NCDE is defined by
\begin{equation}
\label{eq:lin_ncde}
\begin{aligned}
h_{t_0} &= L^1_\theta X_{t_0} + b_\theta, \\
h_t &= h_{t_0} + \int_{t_0}^t A_\theta h_s\,\mathrm{d}X_s, \\
Y_t^\theta(X) &= R_\theta(h_t).
\end{aligned}
\end{equation}
\end{definition}

When $R_\theta$ is linear, Linear NCDEs are universal for terminal-time path-to-point functions \citep{cirone2024deepSSM}. Allowing $R_\theta$ to be non-linear extends this to causal path-to-path functions.

\begin{definition}[Causal path-to-path functions]
\label{def:causal_path_map}
Let
\begin{equation}
T:\mathcal{X}(d_X)\to C([t_0,t_n],\mathbb{R}^{d_y}).
\end{equation}
We say that $T$ is causal if, for every $t\in[t_0,t_n]$ and every $X,\widetilde X\in\mathcal{X}(d_X)$,
\begin{equation}
X|_{[t_0,t]}=\widetilde X|_{[t_0,t]}
\quad\Longrightarrow\quad
T(X)_t=T(\widetilde X)_t.
\end{equation}
\end{definition}

The following theorem is obtained by combining the scalar-valued path-to-path argument of \citet[Proposition~5.1 and Proposition~D.2]{cirone2024deepSSM} with the homogeneous Linear NCDE maximal-expressivity framework of \citet{walker2025structuredlinearcdesmaximally}, after the obvious affine rescaling of time from $[t_0,t_n]$ to $[0,1]$, and then applying the resulting scalar statement coordinate-wise, with scalar tolerance $\varepsilon/\sqrt{d_y}$. Although some of the intermediate Linear CDE statements in \citet{cirone2024deepSSM} allow an additive controlled term, this does not change the present homogeneous formulation, as any affine linear controlled system can be written as a homogeneous linear controlled system after augmenting the hidden state by a coordinate initialised to one and assigning that coordinate zero dynamics. Since paths in $\mathcal{X}(d_X)$ are time-augmented, time is included as a channel of $X$. 

\begin{theorem}[Universality for causal path-to-path functions]
\label{thm:lin_ncde_path_to_path}
Let
\begin{equation}
T:\mathcal{X}(d_X)\to C([t_0,t_n],\mathbb{R}^{d_y})
\end{equation}
be continuous and causal, and let $\mathcal{K}\subset\mathcal{X}(d_X)$ be compact.
Consider Linear NCDEs as in Definition~\ref{def:lin_NCDE}, where $R_\theta$ is a feed-forward neural network.
Then for every $\varepsilon>0$ there exist a hidden dimension $d_h\in\mathbb{N}$ and model parameters such that
\begin{equation}
\label{eq:path_to_path_uniform_approx}
\sup_{X\in\mathcal{K}}\rho_\infty\bigl(Y^\theta(X),T(X)\bigr)\le\varepsilon.
\end{equation}
\end{theorem}

Since $C([t_0,t_n],\mathbb{R}^{d_y})$ endowed with $\rho_\infty$ is separable, Theorem~\ref{thm:lin_ncde_path_to_path} and Lemma~\ref{lem:winfty_pointwise_bound} imply the corresponding causal distributional statement. Indeed, if $\mu$ is a Borel probability measure on $\mathcal{X}(d_X)$ satisfying $\mu(\mathcal{K})=1$ for some compact set $\mathcal{K}\subset\mathcal{X}(d_X)$, and if
\begin{equation}
T:\mathcal{X}(d_X)\to C([t_0,t_n],\mathbb{R}^{d_y})
\end{equation}
is continuous and causal, then for every $\varepsilon>0$ there exists a Linear NCDE with a feed-forward readout such that
\begin{equation}
W_\infty\bigl((Y^\theta)_\#\mu,T_\#\mu\bigr)\le\varepsilon,
\end{equation}
where $W_\infty$ is computed on $C([t_0,t_n],\mathbb{R}^{d_y})$ with respect to the metric $\rho_\infty$.

In contrast, the analogous causal path-to-path universality statement is not currently available for standard SSM architectures such as S4 and Mamba. Existing negative results show that single-layer S4 and Mamba with linear readouts are not universal for terminal-time path-to-point functions, which is one of the key ingredients used to obtain the causal path-to-path universality result for Linear NCDEs \citep{cirone2024deepSSM}. Moreover, while universality results do exist for discrete sequence-to-sequence SSMs with non-linear projections or layer-wise nonlinearities \citep{orvieto2024universality, wang2023state}, these results do not give a continuous causal path-to-path approximation theorem with respect to the supremum metric. Thus they do not directly yield the $W_\infty$ pushforward statement used above. To the best of our knowledge, no directly analogous path-to-path universality theorem is known for stacked S4 or Mamba-style models in the continuous path-space setting considered here.

In practice, the model used in this work is conditional rather than unconditional. Given an observed input time series, we first form a conditional Gaussian process and sample a latent trajectory from its conditional law. This latent trajectory is then sampled on a finite time grid and converted into a path in $\mathcal{X}(d_X)$ by a deterministic interpolation and augmentation map
\begin{equation}
\mathcal{I}:C([t_0,t_n],\mathbb{R}^{d_z})\to \mathcal{X}(d_X),
\end{equation}
where $d_z$ is the latent Gaussian process dimension and $d_X$ includes the deterministic augmentation channels. We assume that $\mathcal{I}$ is continuous from the supremum topology on $C([t_0,t_n],\mathbb{R}^{d_z})$ to the $1$-variation topology on $\mathcal{X}(d_X)$, as is the case for fixed-grid piecewise linear interpolation together with fixed deterministic augmentation channels. The map $\mathcal{I}$ includes the time channel and any additional deterministic preprocessing needed to ensure that its image lies in $\mathcal{X}(d_X)$, for example a constant channel when the homogeneous initialisation requires one. In particular, the common-initial-point convention is imposed on the interpolated and augmented paths actually fed to the model, not necessarily on the raw Gaussian process sample paths.

This is important, as sample paths from a latent Gaussian process with an Ornstein--Uhlenbeck kernel are almost surely continuous but not absolutely continuous \citep{doob1942brownian, cheridito2003fractional}. Thus the raw latent law is not supported on the absolutely continuous path space processed by the NCDE. However, the interpolated and augmented latent law $\mathcal{I}_\#\mu$ is supported on $\mathcal{X}(d_X)$, and this is the object that is actually processed by the model. The interpolated path is then passed through a stacked Linear NCDE. This stacked architecture should be viewed as the practical analogue of the approximation mechanism described above. The first Linear NCDE block produces a hidden path, while the later blocks provide additional non-linear transformations of this hidden path. When the architecture includes sufficiently expressive pointwise readouts, this falls directly under the preceding approximation argument. Conditional on a fixed observed time series, the model defines a distribution on output paths by pushing the conditional interpolated latent law forward through this deterministic stacked Linear NCDE map. From this perspective, the conditional Gaussian process provides the source of pathwise randomness, while the stacked Linear NCDE provides the expressive mechanism that reshapes this randomness into the target output process.

A second important point to note is that the compact-support assumption in the exact $W_\infty$ statement above does not hold, since $\mathcal{I}_\#\mu$ is typically not compactly supported. 
However, this is mainly a technical distinction rather than a practical obstruction. 
The original conditional Gaussian process law is a Borel probability measure on $C([t_0,t_n],\mathbb{R}^{d_z})$. Equipped with the supremum norm, this space is Polish, as it is complete by \citet[Proposition~4.13]{folland1999realanalysis}, and separable by the Stone--Weierstrass theorem \citep[Theorem~4.45]{folland1999realanalysis}. Hence the law is tight \citep[Theorem~1.3]{billingsley1999convergence}.
Consequently, for every $\delta>0$ there exists a compact set $\mathcal{K}$ of latent paths with $\mu(\mathcal{K})\geq 1-\delta$. Since $\mathcal{I}$ is continuous, $\mathcal{I}(\mathcal{K})$ is compact in $\mathcal{X}(d_X)$ and
\begin{equation}
(\mathcal{I}_\#\mu)\bigl(\mathcal{I}(\mathcal{K})\bigr)\geq 1-\delta.
\end{equation}
Applying the same approximation argument on $\mathcal{I}(\mathcal{K})$ then shows that, for every $\varepsilon,\delta>0$, and for every continuous causal target map
\begin{equation}
T:\mathcal{X}(d_X)\to C([t_0,t_n],\mathbb{R}^{d_y}),
\end{equation}
one can choose parameters $\theta$ such that
\begin{equation}
(\mathcal{I}_\#\mu)\Bigl(\bigl\{Z\in\mathcal{X}(d_X):\rho_\infty\bigl(F_\theta(Z),T(Z)\bigr)>\varepsilon\bigr\}\Bigr)\leq\delta,
\end{equation}
where $F_\theta$ denotes the path-to-path map induced by the stacked Linear NCDE. In other words, although one does not obtain a global $W_\infty$ statement for a non-compact latent law, the model can still approximate the target transformation arbitrarily well on an arbitrarily high-probability region of latent space. This is the regime that matters in practice, since training and evaluation only ever involve finitely many sampled and discretised trajectories, while the non-compact Gaussian tails correspond to rare excursions rather than a fundamental modelling obstruction.

\section{Proof that SLiCEs are path-to-path universal}
\label{app:proof_theorem_slice_path_to_path_universal}

\begin{proof}
\citet[Proposition~D.2]{cirone2024deepSSM} prove the corresponding
path-to-path result for dense Linear NCDEs by reducing causal path-to-path
approximation to path-to-point approximation on stopped prefixes. The proof
uses dense Linear NCDEs only through their terminal path-to-point approximation
property. Replacing them by a SLiCE class with the same property gives the same
construction.
\end{proof}

\section{Proof of expressivity gap example}\label{app:hardcore_gap}

We introduce the setting of the example in Section~\ref{subsec:example_expressivity_gap} in slightly more detail: Denote by $\mathcal{H}$ the set of binary sequences of length $n$ with no consecutive ones:

\begin{equation}
    \mathcal H_n
    :=
    \bigl\{
        c=(c_1,\dots,c_n)\in\{0,1\}^n
        :
        c_i c_{i+1}=0
        \text{ for all }1\le i\le n-1
    \bigr\}.
\end{equation}
Then we sample $n$ independently and identically distributed binary Bernoulli Random variables $Z_i$
\[
Z_1,\dots,Z_n \stackrel{\mathrm{i.i.d.}}{\sim} \mathrm{Bernoulli}(p),
\qquad p\in[0,1],
\]
and define the target map $C=(C_1,\dots,C_n)$ as the recursion
\begin{equation}\label{eq:target-map}
    C_1 = Z_1,
    \qquad
    C_k = Z_k(1-C_{k-1})
    \quad (2\le k\le n).
\end{equation}
By construction, $C\in\mathcal H_n$ for every realisation and we write
\[
\mu_{n,p}
:=
C_\sharp\bigl(\mathrm{Bernoulli}(p)^{\otimes n}\bigr)
\]
for the induced law on $\mathcal H_n$.

We now prove Proposition~\ref{prop:hard_core_gap_distribution}.
Throughout this proof, sequence space is equipped with
\[
    d_\infty(c,\tilde c)=\max_{1\le k\le n}|c_k-\tilde c_k|.
\]
For the dense selective upper bound we prove a uniform pathwise error bound over
all \(Z\in\{0,1\}^n\), which implies the stated \(W_\infty\) bound by coupling
each generated sequence with the target sequence obtained from the same input
\(Z\).

\paragraph{Dense selective.} First, we prove the dense selective upper bound.
Let
\[
J=
\begin{pmatrix}
0 & -1\\
1 & 0
\end{pmatrix},
\qquad
e_1=
\begin{pmatrix}
1\\
0
\end{pmatrix}.
\]
Choose \(0<\eta<\min\{\varepsilon,1/2\}\), and define affine maps by
\[
A(z)=(1-z)(\log\eta)I_2+z\pi J,
\qquad
\beta(z)=ze_1 .
\]
Set \(h_0=0\), \(w=e_1\), and \(b=0\).
Then
\[
\exp(A(0))=\eta I_2,
\qquad
\exp(A(1))=-I_2 .
\]
Writing \(x_k=e_1^\top h_k\), the output satisfies
\[
x_k=
\begin{cases}
\eta x_{k-1}, & Z_k=0,\\
1-x_{k-1}, & Z_k=1 .
\end{cases}
\]
The target satisfies the same recursion on one-input steps, while on zero-input steps it satisfies \(C_k=0\).
Since \(x_0=0\) and \(0<\eta<1\), induction gives \(x_k\in[0,1]\) for every \(k\).

Let \(E_k=x_k-C_k\).
If \(Z_k=0\), then \(C_k=0\) and
\[
|E_k|=|x_k|=\eta x_{k-1}\leq\eta .
\]
If \(Z_k=1\), then
\[
E_k=(1-x_{k-1})-(1-C_{k-1})=-E_{k-1}.
\]
Since \(E_1=0\), it follows that \(|E_k|\leq\eta\) for every \(k\) and every binary input \(Z\).
Therefore
\[
\max_{Z\in\{0,1\}^n}
\max_{1\leq k\leq n}
|\widehat C_k-C_k|
\leq\eta
<\varepsilon .
\]
Because \(\eta<1/2\), thresholding \(\widehat C_k=x_k\) at \(1/2\) recovers \(C_k\) exactly.

\paragraph{Dense non-selective.}{}
Consider a dense non-selective exact-flow SSM. It has the form
\[
    h_k=Mh_{k-1}+\beta(Z_k),
    \qquad
    \widehat C_k=w^\top h_k+b,
\]
with fixed \(M\), affine \(\beta\), and \(h_0\) fixed independently of the input.
Unrolling gives
\[
    h_k=M^kh_0+\sum_{j=1}^k M^{k-j}\beta(Z_j).
\]
Hence \(\widehat C_1\) is affine in \(Z_1\), and \(\widehat C_2\) is affine in
\((Z_1,Z_2)\). Writing
\[
    Y_{z_1z_2}
    =
    \bigl(\widehat C_1(z_1,z_2),\widehat C_2(z_1,z_2)\bigr),
    \qquad
    (z_1,z_2)\in\{0,1\}^2,
\]
we therefore have the parallelogram identity
\[
    Y_{00}+Y_{11}-Y_{10}-Y_{01}=0.
\]

It remains to turn this pointwise affine obstruction into a distributional
\(W_\infty\) obstruction. The projection onto the first two coordinates is
\(1\)-Lipschitz, so a \(W_\infty\) approximation of the full length-\(n\) law
with error \(r\) would imply a \(W_\infty\) approximation of the projected
two-step laws with error at most \(r\).

For \(p=1/2\), the target two-step law is the uniform law on the labelled
multiset
\[
    T_{00}=(0,0),\qquad
    T_{01}=(0,1),\qquad
    T_{10}=(1,0),\qquad
    T_{11}=(1,0).
\]
Suppose, for contradiction, that the projected generated law is within
\(W_\infty\) distance \(r<1/4\) of this target law. The open
\(d_\infty\)-balls of radius \(r\) around the three distinct target atoms
\((0,0)\), \((0,1)\), and \((1,0)\) are disjoint. Hence the four generated atoms
can be assigned, counting multiplicity, to the target multiset above so that
\[
    d_\infty(Y_{z_1z_2},\widetilde T_{z_1z_2})<r
\]
for some relabelling \(\{\widetilde T_{00},\widetilde T_{01},
\widetilde T_{10},\widetilde T_{11}\}\) of the multiset
\(\{(0,0),(0,1),(1,0),(1,0)\}\).

For any such relabelling,
\[
    \left\|
        \widetilde T_{00}+\widetilde T_{11}
        -\widetilde T_{10}-\widetilde T_{01}
    \right\|_\infty
    \ge 1.
\]
Indeed, the two plus-labelled targets and the two minus-labelled targets form a
partition of the multiset \(\{(0,0),(0,1),(1,0),(1,0)\}\), and their sums cannot
agree. Using the parallelogram identity for the generated atoms,
\[
\begin{aligned}
1
&\le
\left\|
        \widetilde T_{00}+\widetilde T_{11}
        -\widetilde T_{10}-\widetilde T_{01}
\right\|_\infty  \\
&=
\left\|
        (\widetilde T_{00}-Y_{00})
        +(\widetilde T_{11}-Y_{11})
        -(\widetilde T_{10}-Y_{10})
        -(\widetilde T_{01}-Y_{01})
\right\|_\infty  \\
&<
4r,
\end{aligned}
\]
which contradicts \(r<1/4\). Therefore
\[
    W_\infty(\widehat\mu_{n,1/2},\mu_{n,1/2})\ge \frac14.
\]

\paragraph{Diagonal selective.} 
Finally consider a width \(d\) diagonal selective exact-flow SSM.
It can be written as
\[
h_k=D(Z_k)h_{k-1}+\beta(Z_k),
\qquad
D(z)=\operatorname{diag}(\lambda_1(z),\ldots,\lambda_d(z)),
\qquad
\lambda_i(z)>0 .
\]
For \(p=1\), the input is the all-ones sequence almost surely, so both the
target law and the generated law are Dirac measures. Thus a \(W_\infty\) error
strictly smaller than \(1/2\) is equivalent to pointwise approximation of the
all-ones target sequence with error strictly smaller than \(1/2\).
On the all-ones input, this becomes
\[
h_k=D(1)h_{k-1}+\beta(1).
\]
For a coordinate with diagonal value \(\lambda\neq1\), the scalar recursion contributes a constant plus a multiple of \(\lambda^k\).
For a coordinate with \(\lambda=1\), it contributes a constant plus a multiple of \(k\).
Therefore every linear readout on the all-ones input has the form
\[
\widehat C_k
=
a_0+a_1k+\sum_{\lambda\in\Lambda}a_\lambda\lambda^k,
\qquad
\lambda>0,
\qquad
\lambda\neq1,
\qquad
|\Lambda|+\mathbf 1_{\{a_1\neq0\}}\leq d ,
\]
where \(\Lambda\) is the set of distinct non-unit diagonal values that contribute to the readout.

Subtracting \(1/2\) preserves this form.
Writing \(\alpha=\log\lambda\), we have
\[
\widehat C_k-\frac12
=
a_0+a_1k+\sum_{\alpha\in\Gamma}a_\alpha e^{\alpha k},
\qquad
\alpha\in\mathbb R\setminus\{0\},
\qquad
|\Gamma|+\mathbf 1_{\{a_1\neq0\}}\leq d .
\]
Let
\[
f(t)=a_0+a_1t+\sum_{\alpha\in\Gamma}a_\alpha e^{\alpha t}.
\]
This belongs to the real exponential-polynomial space generated by \(1,t\), and \(e^{\alpha t}\) for \(\alpha\in\Gamma\).
This space is an extended complete Chebyshev system, and by the standard zero-counting property of such systems, if \(f\) is not identically zero, then \(f\) has at most
\[
|\Gamma|+\mathbf 1_{\{a_1\neq0\}}
\]
zeros on any interval \citep{aldaz2009bernstein}.
Since \(\widehat C_k-1/2=f(k)\), each sign change of the nonzero sampled sequence \(\widehat C_1-1/2,\ldots,\widehat C_n-1/2\) gives a zero of \(f\) in the corresponding interval \((k,k+1)\).
Hence \(\widehat C_k-1/2\) has at most \(d\) sign changes.

On the all-ones input, the target is
\[
C_k=
\begin{cases}
1, & k \text{ odd},\\
0, & k \text{ even}.
\end{cases}
\]
Thus \(C_k-1/2\) changes sign at every step.
If the diagonal selective model approximated the target with error strictly less than \(1/2\), then \(\widehat C_k-1/2\) would be nonzero and would have the same sign as \(C_k-1/2\) for every \(k\).
It would therefore have \(n-1\) sign changes.
This contradicts the bound of at most \(d\) sign changes whenever \(n-1>d\).
Hence no width \(d\) diagonal selective exact-flow SSM can approximate the target uniformly with error strictly less than \(1/2\) when \(n\geq d+2\).

This proves
\[
    W_\infty(\widehat\mu_{n,1},\mu_{n,1})\ge \frac12
\]
for width \(d\) real diagonal selective exact-flow SSMs of the stated form
whenever \(n\ge d+2\).

\section{G-SLiCE as an augmented path-space flow}
\label{app:g-slice-as-path-flow}

Let  $G_\theta:\mathcal{X}(d_X)\to\mathcal{Y}(d_y)$ be a generative SLiCE. Introduce an augmented flow-time state
\begin{equation}
    Z^{(s)}=(U^{(s)},Y^{(s)}),
    \qquad
    U^{(0)}=X,
    \qquad
    Y^{(0)}=0,
\end{equation}
and define
\begin{equation}
    \widetilde F_\theta(s,U,Y)
    =
    \begin{pmatrix}
        0\\
        G_\theta(U)
    \end{pmatrix}.
\end{equation}
Then the augmented path-space flow satisfies
\begin{equation}
    \frac{d}{ds}
    \begin{pmatrix}
        U^{(s)}\\
        Y^{(s)}
    \end{pmatrix}
    =
    \begin{pmatrix}
        0\\
        G_\theta(U^{(s)})
    \end{pmatrix}.
\end{equation}
Hence \(U^{(s)}=X\) for all \(s\in[0,1]\), and
\begin{equation}
    Y^{(1)} = G_\theta(X).
\end{equation}
Thus every direct G-SLiCE generator is the terminal projection of an augmented
path-space flow. Under the mild closure assumption that the chosen SLiCE family
can ignore auxiliary channels and append zero-output channels, the augmented
flow formulation inherits Corollary~\ref{cor:slice_universal_generator}.

\section{Additional results}

\subsection{Ablation on importance of dense blocks}

SLiCEs are more expressive than TSFlow's S4 backbone in two key ways. First, S4 is a non-selective state space model (SSM), meaning that its transition matrices are not state-dependent. Second, S4 only uses diagonal transition matrices with an additive rank-$1$ correction. To study the impact of these design choices on time series forecasting, we repeat the probabilistic forecasting experiments from Section~\ref{sec:experiments_probabilistic_forecasting} for two hidden dimensions $d \in \{16, 128\}$ and compare a purely diagonal parameterisation ($b=1$) against a dense block of size $b=16$. All other hyperparameters are kept fixed as described in Appendix~\ref{app:experimental_details}. Results are reported in Table~\ref{tab:slice-ablation-diagonal-vs-dense}.

Across all configurations, the dense variant outperforms the diagonal one in $9$ out of $16$ cases. Moreover, for $5$ out of $8$ datasets, the best-performing configuration uses a dense block (cf.~Table~\ref{tab:slice-best-hparams}), indicating that increased expressivity is often beneficial.

To quantify this effect more precisely, we compare the relative improvement of SLiCE over TSFlow-Cond.\ (OU) based on the results in Table~\ref{tab:conditional-generation-results}, conditioning on whether the dense block is selected. For dataset $i$, we define the relative improvement in CRPS as
\[
r_i = \frac{\text{TSFlow}_i - \text{SLiCE}_i}{|\text{TSFlow}_i|}.
\]

We find that, on datasets where SLiCE uses only diagonal blocks, the average relative improvement is $1.75\%$, whereas it increases to $6.43\%$ on datasets where the dense block is selected. This corresponds to an absolute difference of $4.68$ percentage points. 

Treating dense block usage as a binary variable, the Pearson correlation between this indicator and the relative improvement is $\rho = 0.21$, indicating a positive association between selecting the dense block and achieving larger gains. The magnitude of this effect should be interpreted with caution given the small number of datasets.

\begin{table}[t]
  \centering
  \small
  \caption{Ablation on the diagonal vs (diagonal-)dense versions.}
  \label{tab:slice-ablation-diagonal-vs-dense}
  \resizebox{\textwidth}{!}{%
      \begin{tabular}{lcccccccc}
      \toprule
      Variant & Electricity & Exchange & KDDCup & M4 Hourly & Solar & Traffic & Uber & Wiki2000 \\
      \midrule
      $d=16$, diagonal $b=1$ & $\mathbf{0.0446}$ & $0.0082$ & $0.2871$ & ${0.0236}$ & $\mathbf{0.3260}$ & $0.0830$ & $\mathbf{0.1537}$ & $0.2669$ \\
        $d=16$, dense $b=16$ & $0.0453$ & $\mathbf{0.0079}$ & $\mathbf{0.2866}$ & $\mathbf{0.0225}$ & $0.3476$ & $\mathbf{0.0824}$ & $0.1558$ & $\mathbf{0.2331}$ \\
        \midrule
        $d=128$, diagonal $b=1$ & $\mathbf{0.0435}$ & $0.0078$ & $0.2652$ & $\mathbf{0.0298}$ & $\mathbf{0.3585}$ & $\mathbf{0.0799}$ & $0.1504$ & $0.2293$ \\
        $d=128$, dense $b=16$ & $0.0442$ & $\mathbf{0.0073}$ & $\mathbf{0.2593}$ & $0.0319$ & $0.3855$ & $0.0804$ & $\mathbf{0.1482}$ & $\mathbf{0.2171}$ \\
      \bottomrule
      \end{tabular}%
  }
\end{table}


\subsection{Additional results for irregular sampling generalisation}\label{app:additional_results_irregular_sampling}

Table~\ref{tab:ett15m_cross_irregular_crps_allk} shows that G-SLiCE remains stable across changes in the sampling irregularity at test time, while TSFlow exhibits substantially larger variance and several degradation regimes, particularly for small \(k_{\mathrm{train}}\) and under regular-grid evaluation; representative forecasts for \(k_{\mathrm{train}}=1\) are shown in Figure~\ref{fig:example_forecasts_irr_sampling_all_k_1}.

\begin{table}[t!]
    \centering
    \caption{Cross-irregular generalisation on the ETT 15-minute dataset. Columns indicate the training
irregularity \(k_{\mathrm{train}}\), and rows indicate the evaluation irregularity \(k_{\mathrm{test}}\).
We report CRPS as mean \(\pm\) std. dev. over random seeds. The final row per model reports evaluation
on the regular grid, i.e. \(k_{\mathrm{test}} \to \infty\).}
\label{tab:ett15m_cross_irregular_crps_allk}

    \vspace{0.3em}
    \setlength{\tabcolsep}{3pt}
    \renewcommand{\arraystretch}{0.95}

    \resizebox{0.96\textwidth}{!}{%
    \begin{tabular}{llcccccc}
        \toprule
        & & \multicolumn{6}{c}{\(k_{\mathrm{train}}\)} \\
        \cmidrule(lr){3-8}
        Model & Evaluation & \(1\) & \(3\) & \(10\) & \(25\) & \(50\) & \(100\) \\
        \midrule
        \multirow{7}{*}{G-SLiCE}
        & \(k_{\mathrm{test}} = 1\)
        & \(0.13 \pm 0.01\) & \(0.13 \pm 0.01\) & \(0.16 \pm 0.03\)
        & \(0.14 \pm 0.02\) & \(0.21 \pm 0.15\) & \(0.62 \pm 0.95\) \\
        & \(k_{\mathrm{test}} = 3\)
        & \(0.13 \pm 0.02\) & \(0.12 \pm 0.01\) & \(0.13 \pm 0.02\)
        & \(0.12 \pm 0.01\) & \(0.14 \pm 0.07\) & \(0.28 \pm 0.28\) \\
        & \(k_{\mathrm{test}} = 10\)
        & \(0.15 \pm 0.03\) & \(0.13 \pm 0.01\) & \(0.13 \pm 0.01\)
        & \(0.12 \pm 0.01\) & \(0.13 \pm 0.02\) & \(0.15 \pm 0.01\) \\
        & \(k_{\mathrm{test}} = 25\)
        & \(0.15 \pm 0.03\) & \(0.13 \pm 0.01\) & \(0.13 \pm 0.01\)
        & \(0.12 \pm 0.02\) & \(0.12 \pm 0.01\) & \(0.14 \pm 0.02\) \\
        & \(k_{\mathrm{test}} = 50\)
        & \(0.15 \pm 0.03\) & \(0.13 \pm 0.01\) & \(0.13 \pm 0.01\)
        & \(0.12 \pm 0.02\) & \(0.13 \pm 0.01\) & \(0.13 \pm 0.02\) \\
        & \(k_{\mathrm{test}} = 100\)
        & \(0.15 \pm 0.03\) & \(0.13 \pm 0.01\) & \(0.13 \pm 0.01\)
        & \(0.12 \pm 0.01\) & \(0.13 \pm 0.01\) & \(0.13 \pm 0.02\) \\
        & \(k_{\mathrm{test}} \to \infty\)
        & \(0.15 \pm 0.03\) & \(0.14 \pm 0.01\) & \(0.14 \pm 0.01\)
        & \(0.12 \pm 0.01\) & \(0.13 \pm 0.01\) & \(0.13 \pm 0.02\) \\
        \midrule
        \multirow{7}{*}{TSFlow}
        & \(k_{\mathrm{test}} = 1\)
        & \(0.20 \pm 0.07\) & \(0.66 \pm 1.04\) & \(0.15 \pm 0.05\)
        & \(0.20 \pm 0.10\) & \(0.26 \pm 0.18\) & \(0.19 \pm 0.05\) \\
        & \(k_{\mathrm{test}} = 3\)
        & \(0.17 \pm 0.06\) & \(0.62 \pm 0.93\) & \(0.13 \pm 0.02\)
        & \(0.15 \pm 0.05\) & \(0.22 \pm 0.14\) & \(0.17 \pm 0.04\) \\
        & \(k_{\mathrm{test}} = 10\)
        & \(0.18 \pm 0.06\) & \(0.37 \pm 0.39\) & \(0.13 \pm 0.03\)
        & \(0.17 \pm 0.07\) & \(0.23 \pm 0.16\) & \(0.19 \pm 0.05\) \\
        & \(k_{\mathrm{test}} = 25\)
        & \(0.19 \pm 0.07\) & \(0.34 \pm 0.32\) & \(0.14 \pm 0.03\)
        & \(0.17 \pm 0.07\) & \(0.23 \pm 0.16\) & \(0.19 \pm 0.05\) \\
        & \(k_{\mathrm{test}} = 50\)
        & \(0.19 \pm 0.07\) & \(0.35 \pm 0.35\) & \(0.14 \pm 0.03\)
        & \(0.17 \pm 0.06\) & \(0.25 \pm 0.18\) & \(0.20 \pm 0.06\) \\
        & \(k_{\mathrm{test}} = 100\)
        & \(0.19 \pm 0.07\) & \(0.36 \pm 0.33\) & \(0.15 \pm 0.04\)
        & \(0.18 \pm 0.07\) & \(0.27 \pm 0.21\) & \(0.21 \pm 0.06\) \\
        & \(k_{\mathrm{test}} \to \infty\)
        & \(0.20 \pm 0.07\) & \(0.40 \pm 0.40\) & \(0.17 \pm 0.06\)
        & \(0.19 \pm 0.08\) & \(0.29 \pm 0.22\) & \(0.22 \pm 0.06\) \\
        \bottomrule
    \end{tabular}%
    }
\end{table}

\begin{figure}[t]
    \centering

    \begin{subfigure}{0.48\textwidth}
        \centering
        \includegraphics[width=\linewidth]{Figures/comparison_k1__iid_vs_regular.pdf}
        \caption{$k_\textrm{train}=1$}
        \label{fig:example_forecasts_irr_sampling_all_k_1}
    \end{subfigure}
    \hfill
    \begin{subfigure}{0.48\textwidth}
        \centering
        \includegraphics[width=\linewidth]{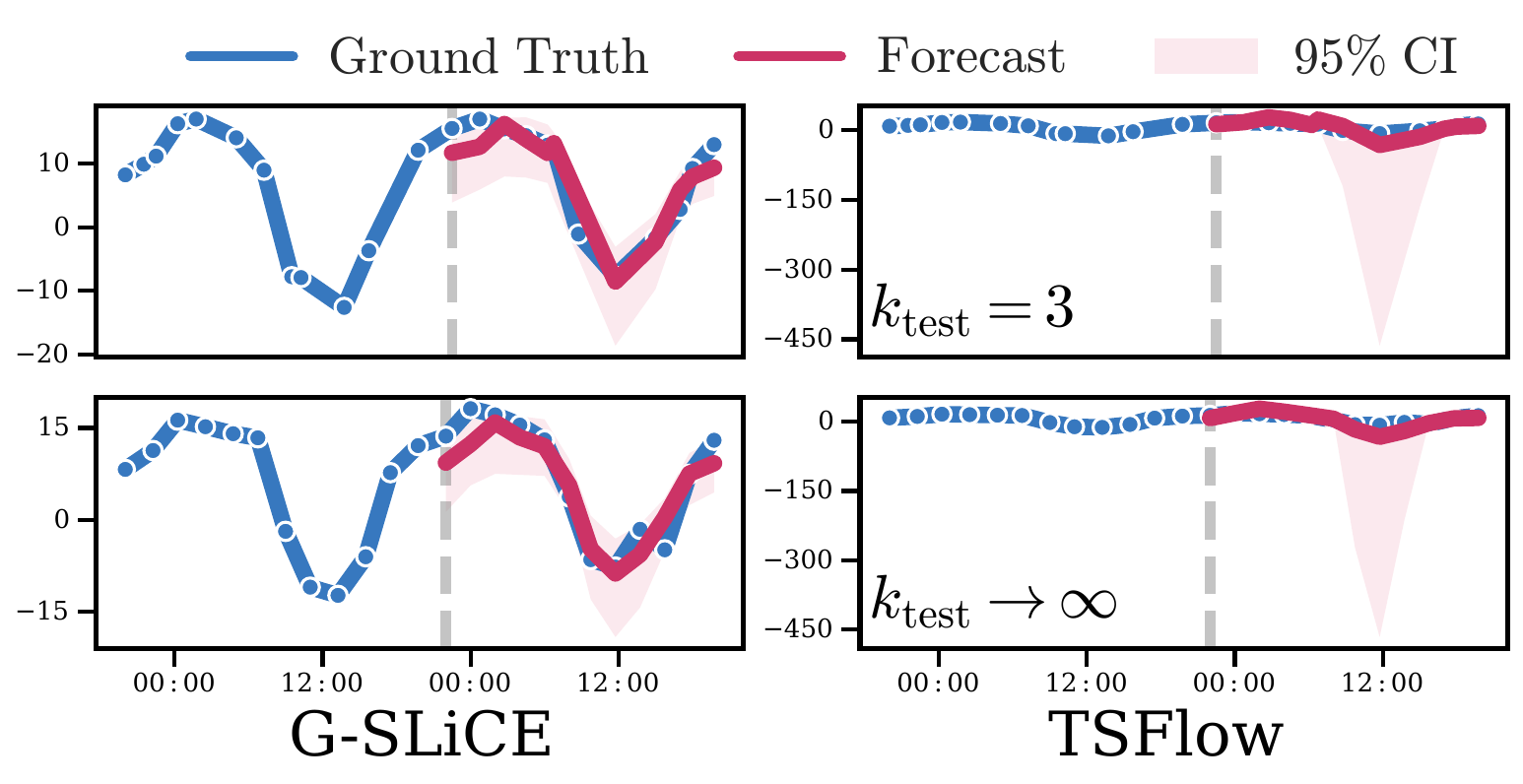}
        \caption{$k_\textrm{train}=3$}
        \label{fig:example_forecasts_irr_sampling_all_k_3}
    \end{subfigure}

    \vspace{0.6em}

    \begin{subfigure}{0.48\textwidth}
        \centering
        \includegraphics[width=\linewidth]{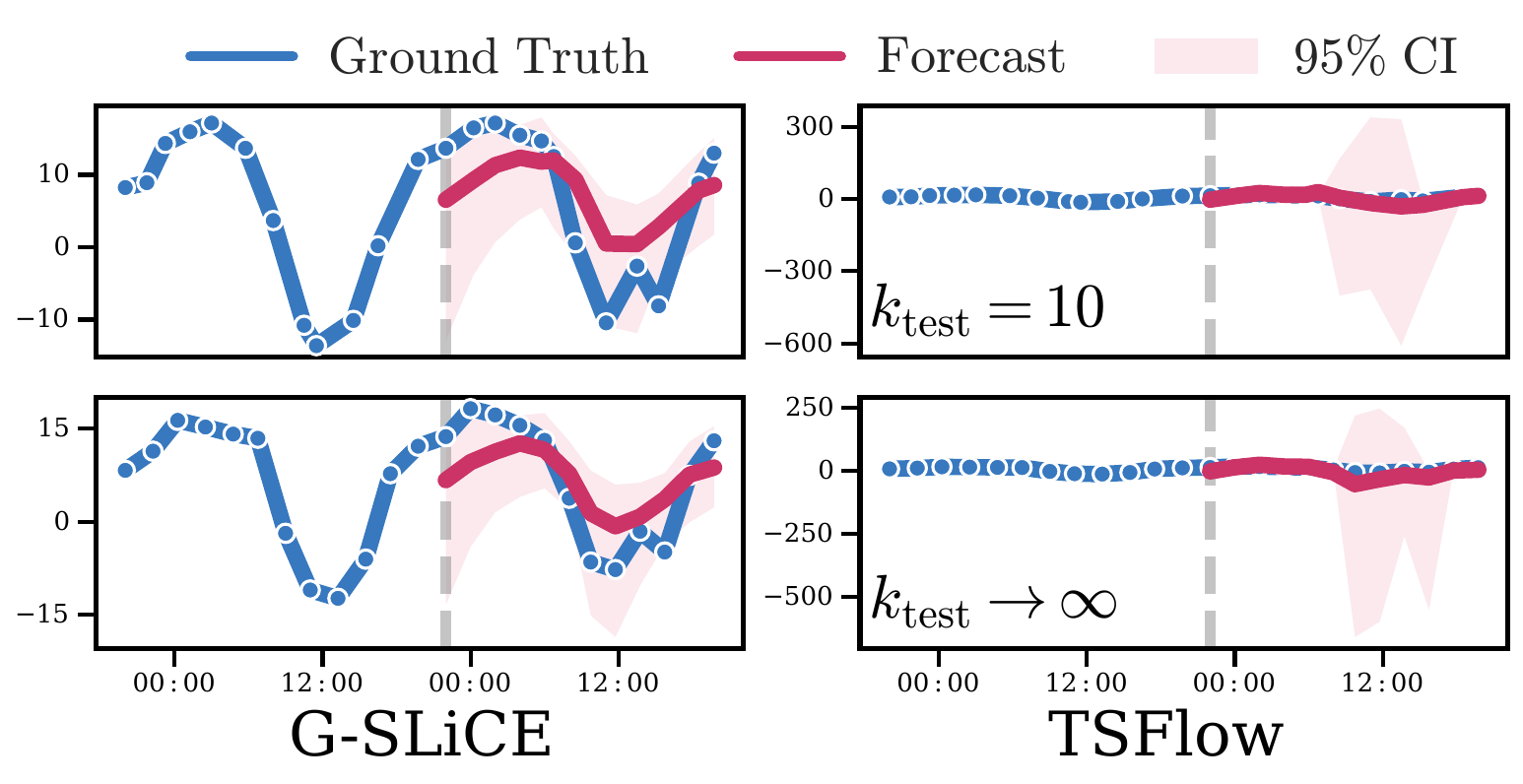}
        \caption{$k_\textrm{train}=10$}
        \label{fig:example_forecasts_irr_sampling_all_k_10}
    \end{subfigure}
    \hfill
    \begin{subfigure}{0.48\textwidth}
        \centering
        \includegraphics[width=\linewidth]{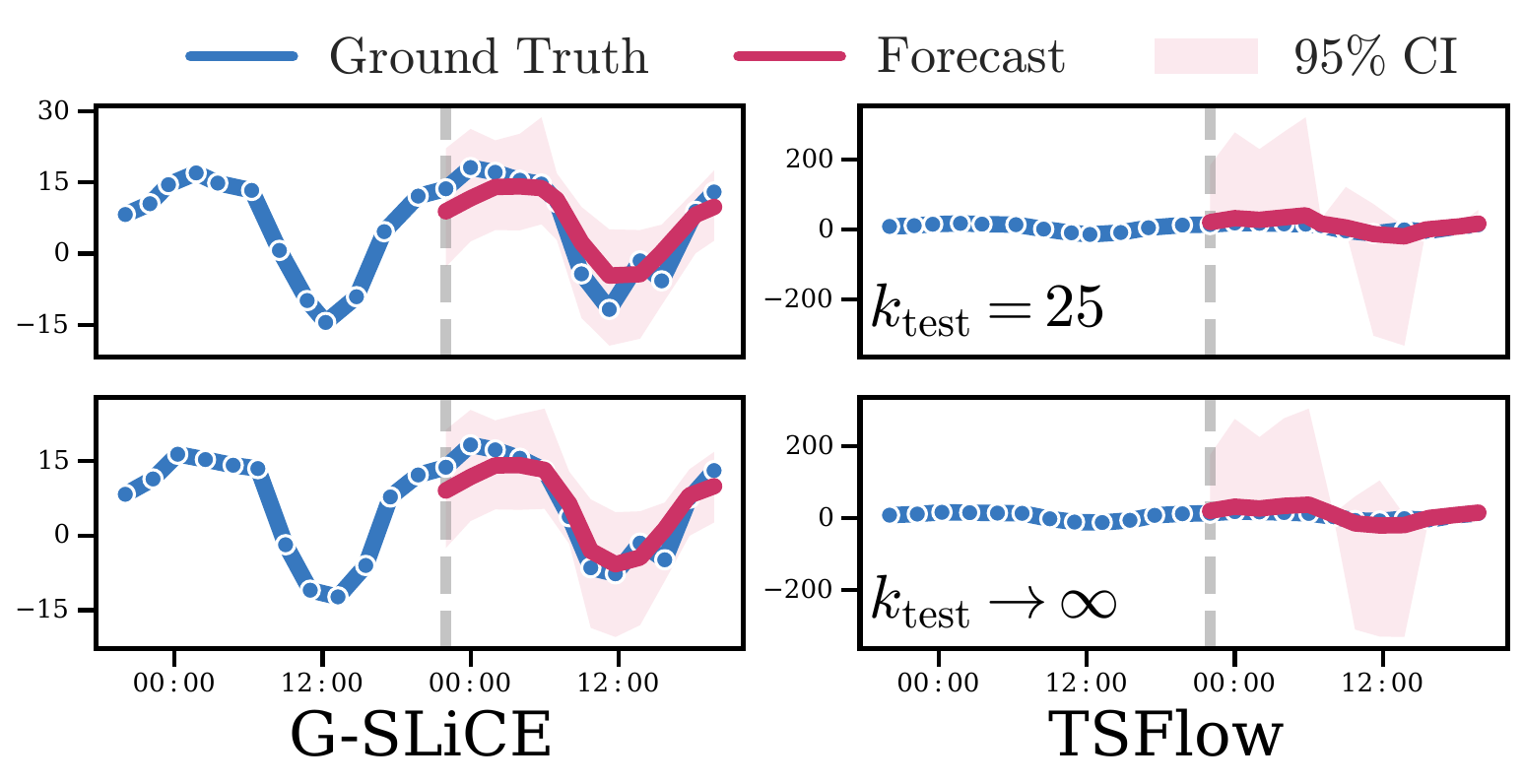}
        \caption{$k_\textrm{train}=25$}
        \label{fig:example_forecasts_irr_sampling_all_k_25}
    \end{subfigure}

    \vspace{0.6em}

    \begin{subfigure}{0.48\textwidth}
        \centering
        \includegraphics[width=\linewidth]{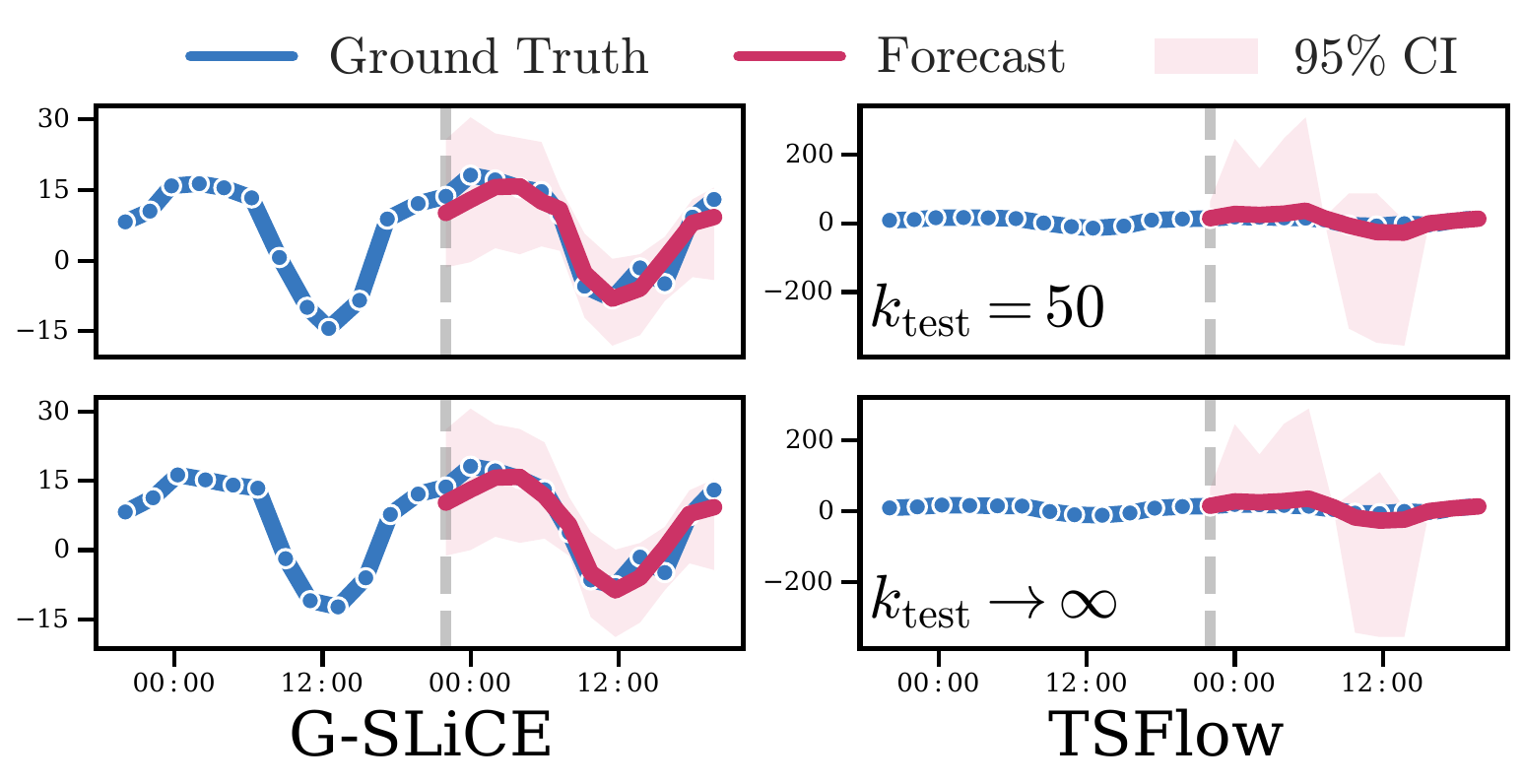}
        \caption{$k_\textrm{train}=50$}
        \label{fig:example_forecasts_irr_sampling_all_k_50}
    \end{subfigure}
    \hfill
    \begin{subfigure}{0.48\textwidth}
        \centering
        \includegraphics[width=\linewidth]{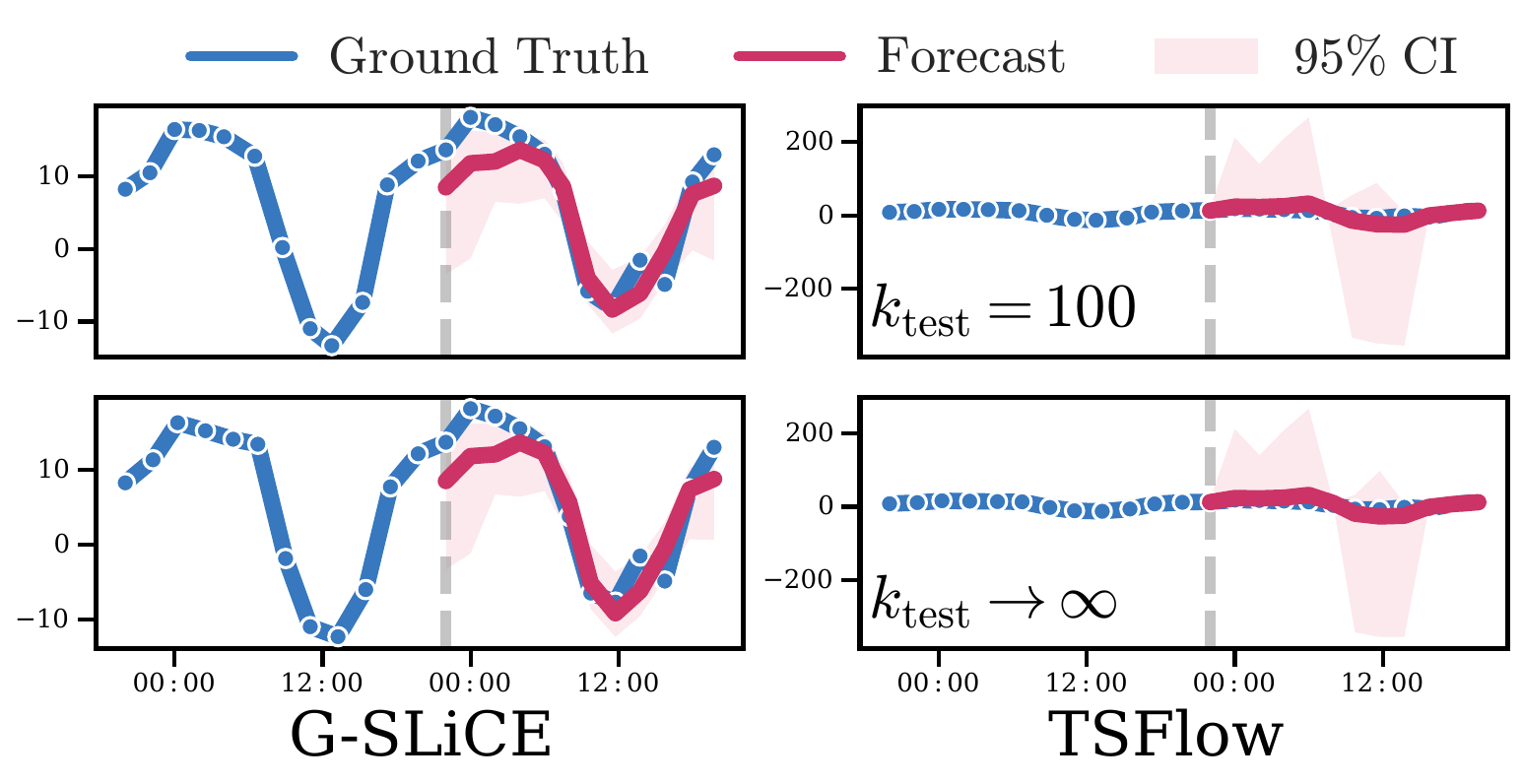}
        \caption{$k_\textrm{train}=100$}
        \label{fig:example_forecasts_irr_sampling_all_k_100}
    \end{subfigure}

    \caption{Example forecasts for G-SLiCEs (left) and TSFlow (right) trained with \(k_{\mathrm{train}} \in \{1,3,10,25,50,100\}\) and evaluated for \(k_{\mathrm{test}}=k_\textrm{train}\) (top) or \(k_{\mathrm{test}}\to\infty\) (bottom). Curves show ground truth and predictive means; shaded regions indicate \(95\%\) confidence intervals.}
    \label{fig:example_forecasts_irr_sampling_all}
\end{figure}

\section{Experimental details}\label{app:experimental_details}

\subsection{Expressivity Gap: Hard-core Example}\label{app:experiments_example_expressivity_gap}

We empirically test the hard-core example from Section~\ref{subsec:example_expressivity_gap}. For each sequence length \(n\in\{8,32,128,512\}\), we sample
\[
    (Z_1,\dots,Z_n)\sim\mathrm{Bernoulli}(p)^{\otimes n}
\]
and train each model to approximate the target map \(C\) defined in Equation~\ref{eq:hardcore_target}. We compare three architectures: G-SLiCE, a diagonal selective SSM, and a dense non-selective SSM. All models are trained with pointwise MSE loss on \(1000\) training samples, with \(200\) validation samples and \(200\) test samples. We train for \(50\) epochs using Adam optimizer, learning rate \(10^{-2}\), batch size \(256\), and no weight decay. Results are averaged over \(3\) random seeds. G-SLiCE uses hidden dimension \(d=2\), while the diagonal selective SSM and dense non-selective SSM use hidden dimension \(d=8\).

Table~\ref{tab:hardcore-results-table} reports two metrics. Exact accuracy is the fraction of test sequences for which the model output matches the full target sequence \(C\). The validity ratio is the fraction of generated sequences lying in \(\mathcal H_n\), i.e. satisfying the no-consecutive-ones constraint. The results agree with the theoretical separation: G-SLiCE learns the hard-core map at all tested sequence lengths, whereas the diagonal selective SSM and dense non-selective SSM fail to recover the exact map beyond the shortest cases. The diagonal selective SSM can still produce valid sequences, but this validity is often degenerate (predicts all-zero sequences) rather than target-tracking. The dense non-selective SSM fits the shortest setting but deteriorates rapidly as \(n\) grows.

For the visual diagnostics, we train the models on sequences of length \(128\). We then evaluate the learned maps on two representative length-\(8\) binary prefixes,
\[
    Z^{(1)}=(1,1,1,1,1,1,1,0),
    \qquad
    Z^{(2)}=(1,1,1,1,0,1,0,1),
\]
and compare the corresponding outputs with the hard-core target. We also estimate the pushforward distributions by drawing \(10^6\) Bernoulli input sequences of length \(128\) and plotting the empirical distributions of their cumulative sums. Figure~\ref{fig:hardcore-results} shows the dense non-selective SSM comparison used in the main text, while Figures~\ref{fig:hardcore-results-trace-both-models} and~\ref{fig:hardcore-results-hist-both-models} provide the corresponding comparisons for both the dense non-selective SSM and the diagonal selective SSM.
\textcolor{blue}{Here, G-SLiCEs uses as hidden dimension of $d=2$, while the diagonal selective and dense non-selective SSM use $d=4$.}

\begin{table}
    \centering
    \caption{
    Validity ratio and exact accuracy on the hard-core target map for the diagonal selective SSM, dense non-selective SSM, and G-SLiCE, evaluated at sequence lengths \(n\in\{8,32,128,512\}\). The validity ratio is the fraction of generated sequences in \(\mathcal H_n\); exact accuracy is the fraction of sequences matching the full target map \(C\).
    }
    \label{tab:hardcore-results-table}
    \begin{tabular}{lcccccccc}
            \toprule
            & \multicolumn{4}{c}{Validity ratio}
            & \multicolumn{4}{c}{Exact accuracy} \\
            \cmidrule(lr){2-5}\cmidrule(lr){6-9}
            Model
            & \(8\) & \(32\) & \(128\) & \(512\)
            & \(8\) & \(32\) & \(128\) & \(512\) \\
            \midrule
            Diag. selective
                & 0.77 & 0.91 & 0.61 & 0.67
                & 0.02 & 0.00 & 0.00 & 0.00 \\
            Dense non-selective 
                & 1.00 & 0.75 & 0.29 & 0.00
                & 1.00 & 0.00 & 0.00 & 0.00 \\
            \midrule
            G-SLiCEs
                & 1.00 & 1.00 & 1.00 & 1.00
                & 1.00 & 1.00 & 1.00 & 1.00 \\
            \bottomrule
        \end{tabular}
\end{table}

\begin{figure}[ht]
    \centering

    \begin{minipage}[c]{0.48\linewidth}
        \centering
        \includegraphics[width=\linewidth]{Figures/slices_expressivity_trace_dense_ssm.pdf}
    \end{minipage}
    \hfill
    \begin{minipage}[c]{0.48\linewidth}
        \centering
        \includegraphics[width=\linewidth]{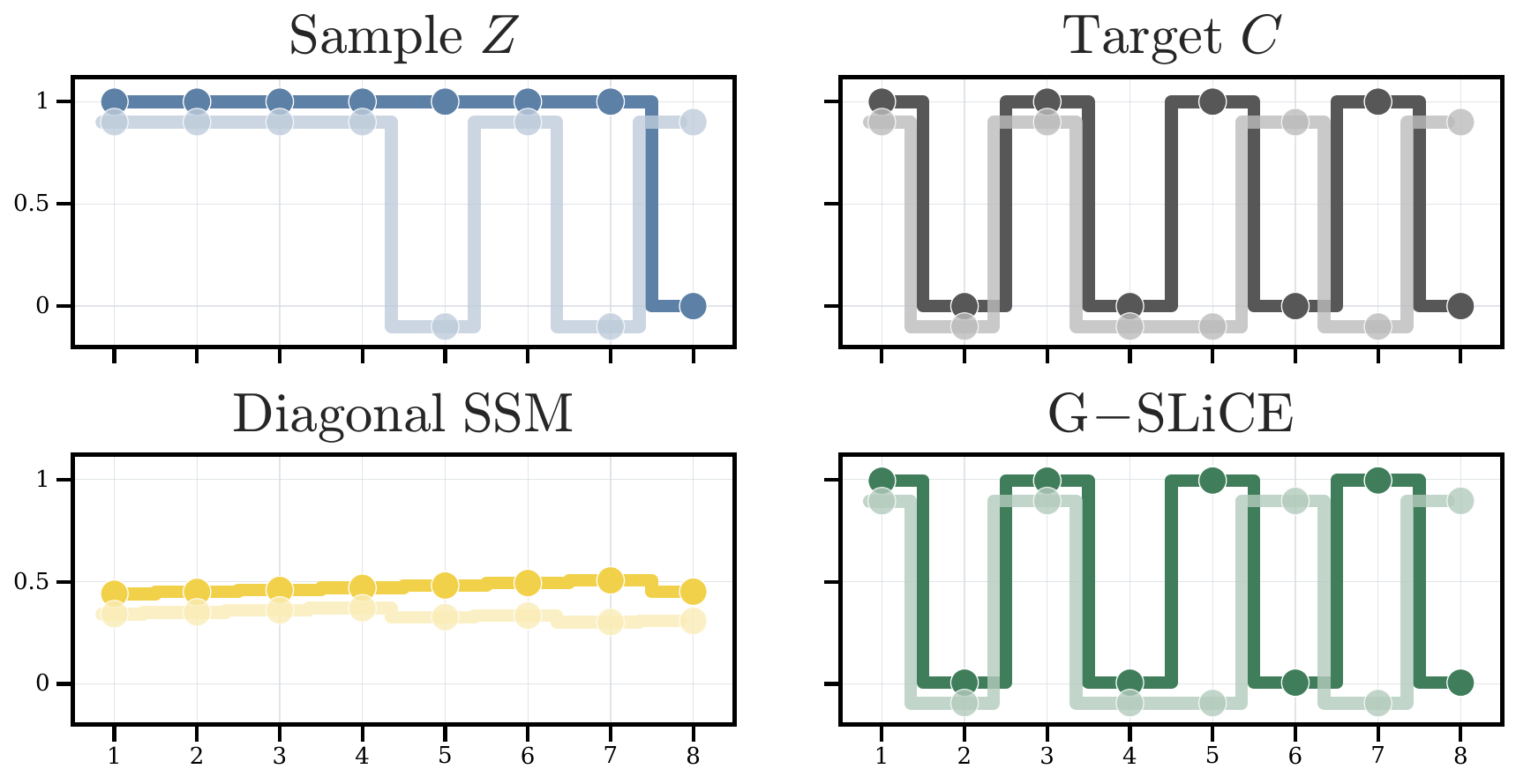}
    \end{minipage}

    \caption{
    Sample-path diagnostics for the hard-core expressivity example. Each block compares a representative Bernoulli input sequence \(Z\), the corresponding hard-core target \(C\), the output of a restricted SSM baseline, and the output of G-SLiCE. Left: dense non-selective SSM. Right: diagonal selective SSM. Both restricted SSMs fail to reproduce the state-dependent target map, while G-SLiCE tracks \(C\) on the displayed examples.
    }
    \label{fig:hardcore-results-trace-both-models}
\end{figure}

\begin{figure}[ht]
    \centering

    \begin{minipage}[c]{0.48\linewidth}
        \centering
        \includegraphics[width=\linewidth]{Figures/slices_expressivity_cumsum_hist_dense_ssm.pdf}
    \end{minipage}
    \hfill
    \begin{minipage}[c]{0.48\linewidth}
        \centering
        \includegraphics[width=\linewidth]{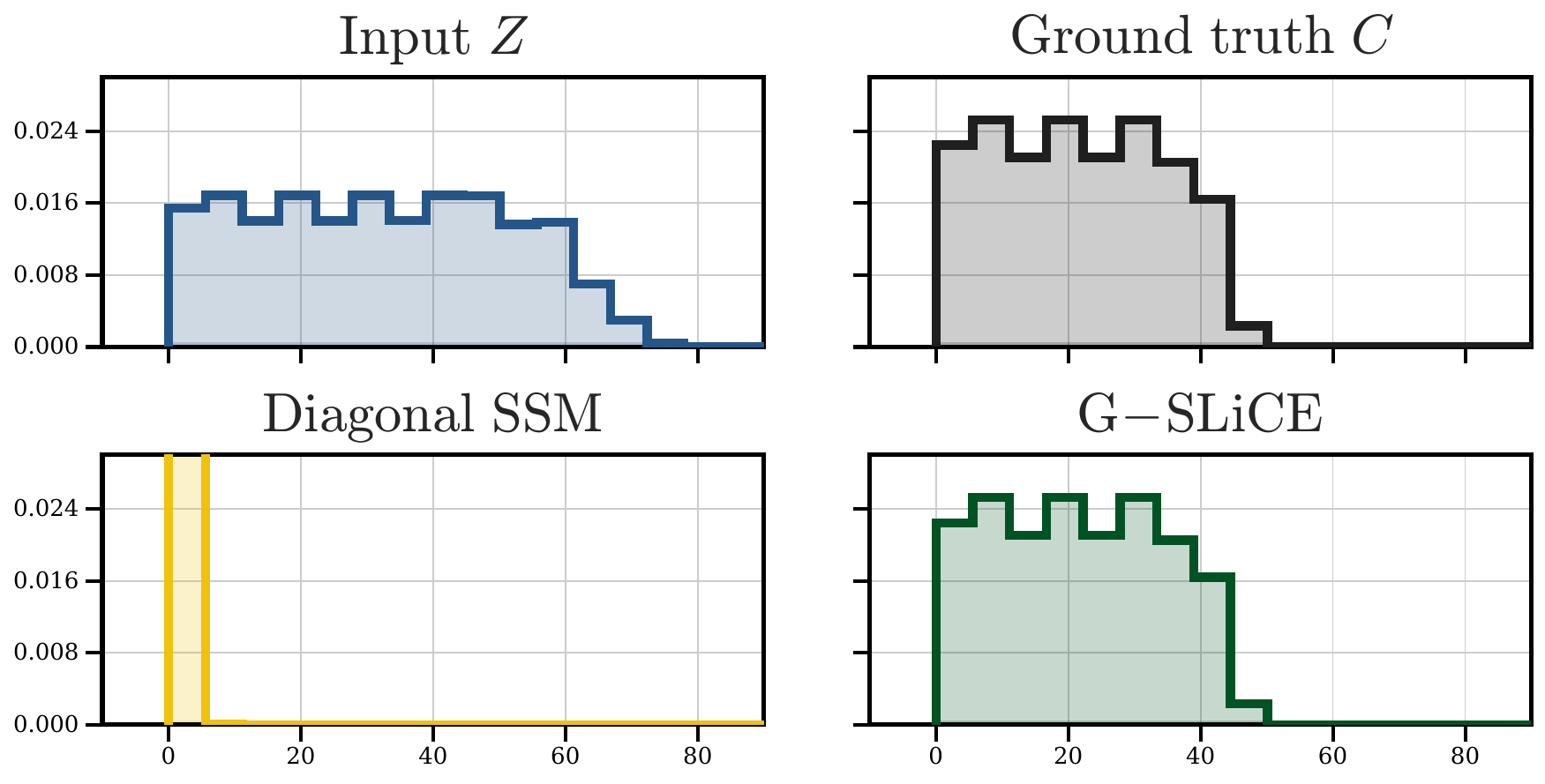}
    \end{minipage}

    \caption{
    Empirical pushforward distributions for the hard-core expressivity example. Each panel shows the distribution of cumulative sums over \(10^6\) generated sequences of length \(128\). Left: dense non-selective SSM compared with the input law, ground-truth hard-core law, and G-SLiCE. Right: the same comparison for the diagonal selective SSM. G-SLiCE matches the ground-truth pushforward distribution, whereas the restricted SSM baselines produce distorted or degenerate distributions.
    }
    \label{fig:hardcore-results-hist-both-models}
\end{figure}

\subsection{Time Series Forecasting: Datasets}

For the conditional and unconditional time series forecasting experiments, we use eight datasets from the GluonTS benchmark collection \citet{gluonts_arxiv,gluonts_jmlr}, following the standard train/validation/test splits provided with each dataset.

The datasets cover a range of value domains, including positive real-valued series ($\mathbb{R}^{+}$), series constrained to the open interval $(0,1)$, and count-valued series ($\mathbb{N}$). They are observed at three different sampling frequencies: 15 minutes (15min), hourly (H), and daily (D). For hourly datasets, we consider forecasting horizons of $24$ or $48$ hours, while for daily datasets we forecast $30$ days ahead. An overview of all datasets and their main characteristics is given in Table~\ref{tab:datasets}.

In the experiments of Sections~\ref{sec:experiments_probabilistic_forecasting} and \ref{sec:experiments_unconditional_generation}, we use the Electricity~\citep{dheeru2017uci}, Exchange~\citep{lai2018modeling}, KDDCup~\citep{godahewa2021monash}, M4~\citep{makridakis2020m4}, Solar~\citep{lai2018modeling}, Traffic~\citep{godahewa2021monash}, UberTLC~\citep{fivethirtyeight2016}, and Wikipedia~\citep{gasthaus2019probabilistic} datasets. For the \emph{grid generalisation} (\S~\ref{sec:experiments_generalising}) experiments, we additionally include the small ETT datasets at $15$-minute and $1$-hour resolutions~\citep{haoyietal-informer-2021}.

\begin{table}[!h]
    \centering
    \caption{Overview of the GluonTS datasets and their statistics used in the experiments in Section \ref{sec:experiments}.}
    \label{tab:datasets}
    \resizebox{\textwidth}{!}{%
    \begin{tabular}{lcccccc}
        \toprule
        Dataset & Train Size & Test Size & Domain & Freq. & Median Seq. Length & Prediction Length \\
        \midrule
        Electricity$^{a}$ \citep{dheeru2017uci}    & 370 & 2590 & $\mathbb{R}^{+}$ & H & 5833 & 24 \\
        Exchange$^{b}$ \citep{lai2018modeling}      & 8   & 40   & $\mathbb{R}^{+}$ & D & 6071 & 30 \\
        KDDCup$^{c}$ \citep{godahewa2021monash}     & 270 & 270  & $\mathbb N$ & H & 10850 & 48 \\
        M4 (H)$^{d}$ \citep{makridakis2020m4}       & 414 & 414  & $\mathbb N$ & H & 960 & 48 \\
        Solar$^{e}$ \citep{lai2018modeling}         & 137 & 959  & $\mathbb{R}^{+}$ & H & 7009 & 24 \\
        Traffic$^{f}$ \citep{godahewa2021monash}    & 963 & 6741 & (0,1) & H & 4001 & 24 \\
        UberTLC$^{g}$ \citep{fivethirtyeight2016}   & 262 & 262  & $\mathbb N$ & H & 4320 & 24 \\
        Wikipedia$^{h}$ \citep{gasthaus2019probabilistic} & 2000 & 10000 & $\mathbb N$ & D & 792 & 30 \\
        ETT Small (15min)$^{i}$ \citep{haoyietal-informer-2021} & 14  & 14  & $\mathbb{R}$ & 15min & 69656 & 96 \\
        ETT Small (H)$^{i}$ \citep{haoyietal-informer-2021}     & 14  & 14  & $\mathbb{R}$ & H     & 17396 & 48 \\
        \bottomrule
    \end{tabular}%
    }
    
    \vspace{0.5em}
    \begin{flushleft}
        {\footnotesize
        $^{a}$ \url{https://archive.ics.uci.edu/ml/datasets/ElectricityLoadDiagrams20112014} \\
        $^{b}$ \url{https://github.com/laiguokun/multivariate-time-series-data} \\
        $^{c}$ \url{https://zenodo.org/record/4656756} \\
        $^{d}$ \url{https://github.com/Mcompetitions/M4-methods/tree/master/Dataset} \\
        $^{e}$ \url{https://www.nrel.gov/grid/solar-power-data.html} \\
        $^{f}$ \url{https://zenodo.org/record/4656132} \\
        $^{g}$ \url{https://github.com/fivethirtyeight/uber-tlc-foil-response} \\
        $^{h}$ \url{https://github.com/mbohlkeschneider/gluon-ts/tree/mv_release/datasets} \\
        $^{i}$ \url{https://github.com/zhouhaoyi/ETDataset}
        }
    \end{flushleft}%
\end{table}

\subsection{Training protocol and hyperparameters}

\paragraph{Data processing.}
We use the preprocessing protocol of \citet{kollovieh2024flow} for all forecasting experiments. Each dataset consists of long univariate or multivariate time series. Training instances are obtained by extracting windows containing a context region and a prediction region. For a dataset with prediction length \(\Delta_{\mathrm{pred}}\), the prediction region has length \(\Delta_{\mathrm{pred}}\), as reported in Table~\ref{tab:datasets}. The context region has the same length.

For conditional forecasting, the observed context values are used to define a conditional Gaussian-process prior. Concretely, for each training instance, we condition the Gaussian process on the context region and sample a path on the full context-plus-prediction grid. This sampled path is used as the stochastic input to the flow model, while the true future values in the prediction region define the target trajectory for conditional flow matching.

The model input is augmented with deterministic feature channels. First, we append a binary observation mask indicating which grid points correspond to observed context values and which grid points are sampled or forecasted. Second, when lag features are enabled, we append historical values from a longer context window. For hourly datasets, the lagged features at time \(t\) are the values from
\(
    1,2,\ldots,7,14,21,28
\)
days before \(t\). For daily datasets, the lagged features are the values from
\(
    1,2,\ldots,7
\)
days before \(t\). These lag channels provide the model with short-term and seasonal historical information without changing the prediction target.

All deterministic augmentations are concatenated channel-wise to the sampled Gaussian-process path before being passed to the vector-field network. The same preprocessing, context construction, Gaussian-process conditioning, masking, lag-feature construction, and train/validation/test splits are used for G-SLiCE and the TSFlow baselines. Additionally, we append the analytical mean of the fitted GP to the features.

\paragraph{Model/Training parameters.} Conditional forecasting models are trained for \(400\) epochs, while unconditional generation models are trained for \(1000\) epochs. Each epoch consists of \(128\) batches with batch size \(64\). We use the Adam optimiser, clip gradients at norm \(0.5\), and maintain an exponential moving average of the model parameters with decay rate \(0.999\).

For conditional experiments, the Gaussian-process prior uses an Ornstein--Uhlenbeck kernel
\[
    k_{\mathrm{OU}}(t,t')
    =
    \exp\!\left(-\frac{|t-t'|}{\ell}\right),
\]
with length scale \(\ell=1\). The Gaussian process is conditioned on the observed context window and sampled on the model input grid. The sampled trajectories are then interpolated and augmented with the deterministic features used by the backbone.

For G-SLiCE, we tune only the architecture-specific hyperparameters and learning rate. The remaining training protocol is kept fixed across datasets and experiments to isolate the effect of replacing the S4-style backbone by a SLiCE backbone. We perform a grid search over bidirectionality, use of lag features, learning rate, number of residual SLiCE blocks, hidden dimension, and dense-block size. The search ranges are shown in Table~\ref{tab:slice-hparam-ranges}. For each dataset, the best configuration is selected by validation CRPS. The selected hyperparameters are reported in Table~\ref{tab:slice-best-hparams} and are used for the corresponding forecasting, unconditional generation, grid-generalisation, and irregular-sampling experiments.

\begin{table}[t]
  \centering
  \small
  \setlength{\tabcolsep}{5pt}
  \caption{Hyperparameter search space for G-SLiCE.}
  \label{tab:slice-hparam-ranges}
  \begin{tabular}{ll}
  \toprule
  Hyperparameter & Range \\
  \midrule
  Bidirectional & \{True, False\} \\
  Use lags & \{True, False\} \\
  Learning rate & \{$0.0001$, $0.001$\} \\
  No. of residual blocks & $\{3, 5\}$ \\
  Hidden dim & $\{16, 64, 128\}$ \\
  Block size & $\{1, 16\}$ \\
  \bottomrule
  \end{tabular}
\end{table}

\begin{table}[t]
  \centering
  \small
  \setlength{\tabcolsep}{5pt}
  \caption{Dataset-specific G-SLiCE hyperparameters selected by validation CRPS.}
  \label{tab:slice-best-hparams}
  \resizebox{\textwidth}{!}{%
      \begin{tabular}{lcccccccc}
      \toprule
      Hyperparameter & Electricity & Exchange & KDDCup & M4 Hourly & Solar & Traffic & Uber & Wiki2000 \\
      \midrule
      Bidirectional & True & True & False & True & True & True & True & True \\
      Use lags & True & True & True & False & True & True & True & False \\
      Learning rate & $0.001$ & $0.0001$ & $0.001$ & $0.001$ & $0.0001$ & $0.001$ & $0.0001$ & $0.0001$ \\
      No. of residual blocks & $5$ & $3$ & $3$ & $5$ & $3$ & $5$ & $5$ & $5$ \\
      Hidden dim & $128$ & $128$ & $128$ & $16$ & $16$ & $64$ & $128$ & $128$ \\
      Block size & $1$ & $16$ & $16$ & $16$ & $1$ & $1$ & $16$ & $16$ \\
      \bottomrule
      \end{tabular}
  }
\end{table}

\subsection{Metrics}

\textbf{Wasserstein distance.} The \textit{Wasserstein Distance} between two probability measures $\mu$ and $\nu$ is a metric that measures the minimum cost of transforming one distribution into another. Given two probability measures $\mu$ and $\nu$ on the metric space $\mathcal{X}$ and a distance metric $d(x,y)$ between two points $x,y\in\mathcal{X}$, a \textit{coupling} $\gamma$ of $\mu$ and $\nu$ is a joint probability measure on $\mathcal{X}\times\mathcal{X}$ whose marginals recover $\mu$ and $\nu$, i.e.\ $\gamma(A\times\mathcal{X})=\mu(A)$ and $\gamma(\mathcal{X}\times B)=\nu(B)$ for all measurable $A,B\subseteq\mathcal{X}$. Denoting the set of all such couplings $\Gamma(\mu,\nu)$, the Wasserstein distance is defined as
\[
    W_p(\mu,\nu) =
    \left(
    \inf_{\gamma\in\Gamma(\mu,\nu)}
    \int_{\mathcal{X}\times\mathcal{X}}
    d(x,y)^p \, d\gamma(x,y)
    \right)^{1/p}.
\]

In our experiments we compute the 2-Wasserstein distance for which we use the implementation provided in \cite{flamary2021pot, flamary2024pot}. To approximate the optimal transport plan $\gamma^*$ we use $10^7$ iterations.

\textbf{CRPS.} The \textit{Continuous Ranked Probability Score (CRPS)} \citep{gneiting2007strictly} is a proper scoring rule for evaluating probabilistic forecasts of a real-valued target. Let $F$ denote the cumulative distribution function (CDF) of a predictive distribution and let $y \in \mathbb{R}$ be the observed outcome. The CRPS is defined as
\[
\mathrm{CRPS}(F, y)
=
\int_{-\infty}^{\infty}
\bigl(F(z) - \mathbbm{1}\{z \ge y\}\bigr)^2 \, dz,
\]
where $\mathbbm{1}\{\cdot\}$ denotes the indicator function. This can be interpreted as the squared $L^2$ distance between the predictive CDF $F$ and the CDF of a point mass at $y$. The CRPS is a strictly proper scoring rule, meaning that it is minimised in expectation if and only if $F$ coincides with the true data-generating distribution. In contrast to point-wise losses, it evaluates the full predictive distribution and therefore captures both accuracy and calibration.

Using the GluonTS library \citet{gluonts_arxiv,gluonts_jmlr}, we approximate the CRPS via a quantile-based representation. For each forecast horizon step $t$ and each quantile level $q \in \mathcal{Q} = \{0.1,\dots,0.9\}$, we compute the pinball (quantile) loss
\[
\rho_q(y_t,\hat{y}_{t,q}) = \max\bigl(q (y_t - \hat{y}_{t,q}),\; (q-1)(y_t - \hat{y}_{t,q})\bigr).
\]
GluonTS aggregates this loss over the forecast horizon and normalises by the total absolute target mass:
\[
\mathrm{wQL}_q = \frac{2 \sum_t \rho_q(y_t,\hat{y}_{t,q})}{\sum_t |y_t|},
\]
The CRPS is then approximated by averaging over the quantile grid:
\[
\mathrm{CRPS} \approx \frac{1}{|\mathcal{Q}|} \sum_{q \in \mathcal{Q}} \mathrm{wQL}_q.
\]

In our experiments, the predictive distribution is represented by $100$ samples from the model, from which the empirical quantiles $\hat{y}_{t,q}$ are estimated.

\textbf{LPS.} The \textit{Linear predictive Score}, as introduced by \cite{kollovieh2023predict}, aims to measure the consistency between generated synthetic and real samples. To calculate it, we fit a ridge regression model on \(10{,}000\) synthetically generated samples to predict vectors of future values $y_f \in \mathbb{R}^{L_f}$ from a past sequence $y_p \in \mathbb{R}^{L_p}$. The LPS is the CRPS of this model on a test set of real samples. For the linear ridge regression model, we use the implementation provided by \cite{scikit-learn} .


\subsection{Construction of irregular sampling grid}\label{app:construction_irregular_sampling_grid}

In this Section, we outline how we irregularly subsample a regular base grid 
\(
    \mathcal{G}
    =
    \{0,\delta,2\delta,\ldots,T\},
\)
on  \([0,T]\). We first sample a continuous irregular grid using the Gamma renewal process. For a desired number of observed points \(N\), we draw \(N-1\) independent increments
\[
    \Delta \tau_i \sim \mathrm{Gamma}(k,\theta),
    \qquad i=1,\ldots,N-1,
\]
and normalise them to span the full window:
\[
    \widetilde{\Delta \tau}_i
    =
    \frac{\Delta \tau_i}{\sum_{j=1}^{N-1}\Delta \tau_j}T,
    \qquad
    \tau_i
    =
    \sum_{j=1}^{i}\widetilde{\Delta \tau}_j ,
\]
with \(\tau_0=0\). This gives continuous target times
\[
    0=\tau_0 < \tau_1 < \cdots < \tau_{N-1}=T .
\]

Since the data are only observed on the base grid \(\mathcal{G}\), we map each continuous target time to its nearest base-grid index:
\[
    m_i
    =
    \operatorname{round}\!\left(\frac{\tau_i}{\delta}\right),
    \qquad
    t_i
    =
    m_i\delta .
\]
We then remove duplicate indices introduced by rounding and, if necessary, resample the Gamma increments until exactly \(N\) distinct grid indices are obtained:
\[
    0 \leq m_0 < m_1 < \cdots < m_{N-1} \leq M .
\]
The final irregular observation grid is therefore
\[
    \mathcal{G}_{\mathrm{irr}}
    =
    \{m_0\delta,m_1\delta,\ldots,m_{N-1}\delta\}
    \subseteq \mathcal{G}.
\]

This construction preserves the irregular spacing induced by the Gamma renewal process while ensuring that every selected observation corresponds to an actual timestamp in the ETTSmall15min data. 

The shape parameter \(k\) again controls the regularity of the selected grid: \(k=1\) produces bursty, highly irregular subsets, while larger \(k\) yields subsets closer to uniform subsampling of the base grid.

\subsection{Compute resources}

All experiments were run on a compute cluster with NVIDIA A100 ($80$ GB) and H200 ($141$ GB) GPUs. Each node was equipped with a $64$-core AMD EPYC 9334 CPU clocked at $3.90$ GHz and $256$ GB of RAM. 

\subsection{Statistical tests}
\label{app:statistical_tests}

Because \citet{bilovs2023modeling} has a missing entry, it was excluded from the Friedman test. On the remaining models, the Friedman test indicated significant differences in CRPS (statistic of $54.785$, $p$-value of $1.977\times 10^{-7}$). G-SLiCEs achieved the best average rank ($1.812$), followed by TSFlow ($2.312$). For the pairwise comparison between G-SLiCE and TSFlow, the one-sided Wilcoxon signed-rank test on log-ratios gave a statistic of $6.0$ and $p$-value $0.055$. When Wiki2000 was excluded from this Wilcoxon calculation, the test gave a statistic of $0$ and a $p$-value $ 0.008$.

\clearpage

\end{document}